\documentclass[10pt,conference,twocolumn]{ieeeconf}

\IEEEoverridecommandlockouts
\makeatletter
\let\proof\proof

\let\proof\@undefined
\let\endproof\@undefined
\makeatother

\usepackage{makeidx}         
\usepackage{graphicx}        
\usepackage{multicol}        
\usepackage[bottom]{footmisc}
\usepackage{latexsym}

\usepackage[pdftex,bookmarks=true]{hyperref}

\usepackage[caption=false,font=footnotesize]{subfig}
\usepackage{cite,color,comment,xspace}
\usepackage{amsfonts}
\usepackage{amsmath}
\usepackage{amssymb}
\usepackage{amsthm}
\usepackage{algorithm}
\usepackage{algorithmic}
\usepackage{url}
\usepackage{times}
\usepackage{enumerate}
\usepackage{epstopdf}
\usepackage{epsfig}
\usepackage{array}
\usepackage{fixltx2e}

\graphicspath{{fig/}}

\newtheorem{theorem}{Theorem}[section]
\newtheorem{proposition}[theorem]{Proposition}

\newtheorem{definition}[theorem]{Definition}

\newtheorem{remark}[theorem]{Remark}

\newtheorem{problem}[theorem]{Problem}

\newcommand{\ev}{\Diamond}
\newcommand{\gl}{\Box}
\newcommand{\un}{\textrm{ }\mathcal U}

\newcommand{\notltl}{\neg}
\newcommand{\andltl}{\wedge}
\newcommand{\orltl}{\vee}
\newcommand{\nextltl}{\bigcirc}

\newcommand{\bo}{{\bf 1}}
\newcommand{\n}{\nonumber\\}

\newcommand{\be}{\begin{equation}}
\newcommand{\ee}{\end{equation}}
\newcommand{\ben}{\begin{equation*}}
\newcommand{\een}{\end{equation*}}
\newcommand{\bea}{\begin{eqnarray}}
\newcommand{\eea}{\end{eqnarray}}
\newcommand{\bean}{\begin{eqnarray*}}
\newcommand{\eean}{\end{eqnarray*}}
\newcommand{\ba}{\begin{array}}
\newcommand{\ea}{\end{array}}
\newcommand{\leftm}{\left[\begin{array}}
\newcommand{\rightm}{\end{array}\right]}

\newcommand{\ie}{{\it i.e., }}
\newcommand{\eg}{{\it e.g., }}

\DeclareMathOperator*{\argmin}{arg\,min}

\newcommand\oprocendsymbol{\hbox{$\bullet$}}
\newcommand\oprocend{\relax\ifmmode\else\unskip\hfill\fi\oprocendsymbol}

\makeindex             
                       
 \title{MDP Optimal Control under Temporal Logic Constraints \\ - Technical Report -}
\author{Xu Chu Ding \qquad Stephen L. Smith \qquad Calin Belta \qquad Daniela Rus\thanks{This work was supported in part by ONR-MURI N00014-09-1051, ARO W911NF-09-1-0088, AFOSR YIP FA9550-09-1-020, and NSF CNS-0834260.}
\thanks{X. C. Ding and C. Belta are with Department of Mechanical Engineering, Boston University, Boston, MA 02215, USA (email: {\{xcding; cbelta\}@bu.edu)}.   S. L. Smith is with the Department of Electrical and Computer Engineering, University of Waterloo, Waterloo ON, N2L 3G1 Canada (email: stephen.smith@uwaterloo.ca).  D. Rus is with the Computer Science and Artificial Intelligence Laboratory, Massachusetts Institute of Technology, Cambridge, MA 02139, USA  (email: rus@csail.mit.edu).}
}
\begin{document}
\maketitle \thispagestyle{empty} \pagestyle{empty}

\begin{abstract} In this paper, we develop a method to automatically generate a control policy for a dynamical system modeled as a Markov Decision Process (MDP).  The control specification is given as a Linear Temporal Logic (LTL) formula over a set of propositions defined on the states of the MDP.  We synthesize a control policy such that the MDP satisfies the given specification almost surely, if such a policy exists. In addition, we designate an ``optimizing proposition'' to be repeatedly satisfied, and we formulate a novel optimization criterion in terms of minimizing the expected cost in between satisfactions of this proposition. We propose a sufficient condition for a policy to be optimal, and develop a dynamic programming algorithm that synthesizes a policy that is optimal under some conditions, and sub-optimal otherwise. This problem is motivated by robotic applications requiring persistent tasks, such as environmental monitoring or data gathering, to be performed. 
\end{abstract}

\section{Introduction}
\label{sec:intro}

In this paper, we consider the problem of controlling a (finite state) Markov Decision Process (MDP). Such models are widely used in various areas including engineering, biology, and economics. 
In particular, in recent results, they have been successfully used to model and control autonomous robots subject to uncertainty in their sensing and actuation, such as for ground robots \cite{LaWaAnBe-ICRA10}, unmanned aircraft \cite{temizercollision} and surgical steering needles \cite{alterovitz2007stochastic}.


Several authors
\cite{Hadas-ICRA07,Karaman_mu_09,Loizou04,Tok-Ufuk-Murray-CDC09} have
proposed using temporal logics, such as Linear Temporal Logic (LTL)
and Computation Tree Logic (CTL) \cite{Clarke99}, as specification
languages for control systems.  Such logics are appealing because they
have well defined syntax and semantics, which can be easily used to
specify complex behavior. In particular, in LTL, it is possible to
specify persistent tasks, {\it e.g.,} ``Visit regions $A$, then $B$,
and then $C$, infinitely often. Never enter $B$ unless coming directly
from $D$.''  In addition, off-the-shelf model checking
algorithms~\cite{Clarke99} and temporal logic game
strategies~\cite{Piterman-2006} can be used to verify the correctness
of system trajectories and to synthesize provably correct control
strategies.


The existing works focusing on LTL assume that a finite model of the
system is available and the current state can be precisely
determined. If the control model is deterministic ({\it i.e.,} at each
state, an available control enables exactly one transition), control
strategies from specifications given as LTL formulas can be found
through simple adaptations of off-the-shelf model checking algorithms
\cite{KB-TAC08-LTLCon}. If the control is non-deterministic (an
available control at a state enables one of several transitions, and
their probabilities are not known), the control problem from an LTL
specification can be mapped to the solution of a B\"uchi or GR(1) game
if the specification is restricted to a fragment of
LTL~\cite{Hadas-ICRA07,KlBe-HSCC08-book}. If the probabilities of the
enabled transitions at each state are known, the control problem
reduces to finding a control policy for an MDP such that a
probabilistic temporal logic formula is satisfied \cite{de1997formal}.

By adapting methods from probabilistic model-checking
\cite{baier2008principles,de1997formal,vardi1999probabilistic}, we
have recently developed frameworks for deriving MDP control policies
from LTL formulas~\cite{IFAC2011_LTL}, which is related to a number of
other approaches \cite{courcoubetis1998markov,baier2004controller}
that address the problem of synthesizing control policies for MDPs
subject to LTL satisfaction constraints.  In all of the above
approaches, a control policy is designed to maximize the probability
of satisfying a given LTL formula.  However, no attempt has been made
so far to optimize the long-term behavior of the system, while
enforcing LTL satisfaction guarantees.  Such an objective is often
critical in many applications, such as surveillance, persistent
monitoring, and pickup delivery tasks, where a robot must repeatedly
visit some areas in an environment and the time in between revisits
should be minimized.




As far as we know, this paper is the first attempt to compute an
optimal control policy for a dynamical system modeled as an MDP, while
satisfying temporal logic constraints.  This work begins to bridge the
gap between our prior work on MDP control policies maximizing the
probability of satisfying an LTL formula~\cite{IFAC2011_LTL} and
optimal path planning under LTL
constraints~\cite{SLS-JT-CB-DR:10b}. We consider LTL formulas defined
over a set of propositions assigned to the states of the MDP. We
synthesize a control policy such that the MDP satisfies the
specification almost surely, if such a policy exists. In addition, we
minimize the expected cost between satisfying instances of an
``optimizing proposition.''

The main contribution of this paper is two-fold.  First, we formulate the above MDP optimization problem in terms of minimizing the average cost per
cycle, where a cycle is defined between successive satisfaction of the optimizing 
proposition. We present a novel connection between this problem
and the well-known average cost per stage problem. Second, we incorporate the LTL constraints and obtain a
sufficient condition for a policy to be optimal.  We present a dynamic
programming algorithm that under some (heavy) restrictions synthesizes an optimal control policy, and a sub-optimal policy otherwise.

The organization of this paper is as follows. In
Sec.~\ref{sec:prelim} we provide some preliminaries. We formulate the problem 
in Sec.~\ref{sec:problem_formulation} and we formalize the connection between the average cost per cycle and the average cost per stage problem in Sec.~\ref{sec:solutiontoaveragecycle}. 
In Sec.~\ref{sec:solveLTLprob}, we provide a
method for incorporating LTL constraints. We present a case study illustrating our framework in Sec.~\ref{sec:casestudy} and we conclude in Sec.~\ref{sec:conc}



\section{Preliminaries}
\label{sec:prelim}
%

\subsection{Linear Temporal Logic}
We employ Linear Temporal Logic (LTL) to describe MDP control specifications. A detailed description of the syntax and semantics of LTL is beyond the scope of this paper and can be found in \cite{baier2008principles,Clarke99}. 
Roughly, an LTL formula is built up from a set of atomic propositions $\Pi$, standard Boolean operators $\notltl$ (negation), $\orltl$ (disjunction), $\andltl$ (conjunction), and temporal operators $\nextltl$ (next), $\un$ (until), $\ev$ (eventually), $\gl$ (always). The semantics of LTL formulas are given over infinite words in $2^{\Pi}$. A word satisfies an LTL formula $\phi$ if $\phi$ is true at the first position of the word; $\gl \phi$ means that $\phi$ is true at all positions of the word; $\ev \phi$ means that $\phi$ eventually becomes true in the word;  $\phi_{1}\un\phi_{2}$ means that $\phi_1$ has to hold at least until $\phi_2$ is true.  More expressivity can be achieved by combining the above temporal and Boolean operators (more examples are given later).  
An LTL formula can be represented by a deterministic \emph{Rabin automaton}, which is defined as follows.
\begin{definition}[Deterministic Rabin Automaton]
\label{def:DRA}
A deterministic Rabin automaton (DRA) is a tuple $\mathcal R=(Q,\Sigma,\delta,q_{0},F)$, where (i) $Q$ is a finite set of states; (ii) $\Sigma$ is a set of inputs (alphabet); (iii) $\delta:Q\times\Sigma\rightarrow Q$ is the transition function; (iv) $q_{0}\in Q$ is the initial state; and (v) $F=\{(L(1),K(1)),\dots,(L(M),K(M))\}$ is a set of pairs of sets of states such that $L(i),K(i)\subseteq Q$ for all $i=1,\dots,M$.
\end{definition}

A run of a Rabin automaton $\mathcal R$, denoted by $r_{\mathcal R}=q_{0}q_{1}\ldots$, is an infinite sequence of states in $\mathcal R$ such that for each $i\geq 0$, $q_{i+1}\in\delta(q_{i},\alpha)$ for some $\alpha \in \Sigma$.   A run $r_{\mathcal R}$ is {\it accepting} if there exists a pair $(L,K)\in F$ such that 1) there exists $n\geq 0$, such that for all $m\geq n$, we have $q_{m}\notin L$, and 2) there exist infinitely many indices $k$ where $q_{k}\in K$.   
This acceptance conditions means that $r_{\mathcal R}$ is accepting if for a pair $(L,K)\in F$, $r_{\mathcal R}$ intersects with $L$ finitely many times and $K$ infinitely many times.  Note that for a given pair $(L,K)$, $L$ can be an empty set, but $K$ is not empty.

For any LTL formula $\phi$ over $\Pi$, one can construct a DRA with input alphabet $\Sigma= 2^{\Pi}$ accepting all and only words over $\Pi$ that satisfy $\phi$ (see \cite{gradel2002automata}).  We refer readers to \cite{klein2006experiments
} and references therein for algorithms and to freely available implementations, such as \cite{ltl2dstar}, to translate a LTL formula over $\Pi$ to a corresponding DRA.

\subsection{Markov Decision Process and probability measure}
\label{subsec:MDPandprobmeasure}
\begin{definition}[Labeled Markov Decision Process]
\label{def:MDP}
A labeled Markov decision process (MDP) is a tuple $\mathcal M=(S, U, P, s_{0}, \Pi, \mathcal L, g)$, where $S=\{1,\ldots,n\}$ is a finite set of states; $U$ is a finite set of controls (actions) (with slight abuse of notations we also define the function $U(i)$, where $i\in S$ and $U(i)\subseteq U$ to represent the available controls at state $i$); $P: S\times U\times S\rightarrow [0,1]$ is the transition probability function such that for all $i\in S$,
$\sum_{j\in S}P(i,u,j) = 1$ if $u\in U(i)$, and $P(i,u,j)=0$ if $u\notin U(i)$; $s_{0}\in S$ is the initial state; $\Pi$ is a set of atomic propositions; $\mathcal L: S\rightarrow 2^{\Pi}$ is a labeling function and $g:S \times U\rightarrow \mathbb R^{+}$ is such that $g(i,u)$ is the expected (non-negative) cost when control $u\in U(i)$ is taken at state $i$.
\end{definition}

We define a control function $\mu:S\rightarrow U$ such that $\mu(i)\in U(i)$ for all $i\in S$.  A infinite sequence of control functions $M=\{\mu_{0},\mu_{1},\ldots\}$ is called a \emph{policy}.  One can use a policy to resolve all non-deterministic choices in an MDP by applying the action $\mu_{k}(s_{k})$ at state $s_{k}$.  Given an initial state $s_{0}$, an infinite sequence $r^{M}_{\mathcal M}=s_{0}s_{1}\ldots$ on $\mathcal M$ generated under a policy $M$ is called a \emph{path} on $\mathcal M$ if $P(s_{k},\mu_{k}(s_{k}),s_{k+1})>0$ for all $k$.  The index $k$ of a path is called \emph{stage}.  
If $\mu_{k}=\mu$ for all $k$, then we call it a \emph{stationary} policy and we denote it simply as $\mu$.  A stationary policy $\mu$ induces a Markov chain where its set of states is $S$ and the transition probability from state $i$ to $j$ is $P(i,\mu(i),j)$.


We define $\mathrm{Paths}^{M}_{\mathcal M}$ and $\mathrm{FPaths}^{M}_{\mathcal M}$ as the set of all infinite and finite paths of $\mathcal M$ under a policy $M$, respectively.  
We can then define a probability measure over the set $\mathrm{Paths}^{M}_{\mathcal M}$.  For a path $r_{\mathcal M}^{M}=s_{0}s_{1}\ldots s_{m}s_{m+1}\ldots \in\mathrm{Paths}^{M}_{\mathcal M}$, the \emph{prefix} of length $m$ of $r_{\mathcal M}^{M}$ is the finite subsequence $s_{0}s_{1}\ldots s_{m}$.  Let $\mathrm{Paths}^{M}_{\mathcal M}(s_{0}s_{1}\ldots s_{m})$ denote the set of all paths in $\mathrm{Paths}^{M}_{\mathcal M}$ with the prefix $s_{0}s_{1}\ldots s_{m}$.  (Note that $s_{0}s_{1}\ldots s_{m}$ is a finite path in $\mathrm{FPaths}^{M}_{\mathcal M}$.)  Then, the probability measure $\textrm{Pr}^{M}_{\mathcal M}$ on the smallest $\sigma$-algebra over $\mathrm{Paths}^{M}_{\mathcal M}$ containing $\mathrm{Paths}^{M}_{\mathcal M}(s_{0}s_{1}\ldots s_{m})$ for all $s_{0}s_{1}\ldots s_{m}\in \mathrm{FPaths}^{M}_{\mathcal M}$ is the unique measure satisfying:
\bean
\textrm{Pr}^{M}_{\mathcal M}\{\mathrm{Paths}^{M}_{\mathcal M}(s_{0}s_{1}\ldots s_{m})\}=\prod_{0\leq k < n} P(s_{k},\mu_{k}(s_{k}),s_{k+1}).
\eean

Finally, we can define the probability that an MDP $\mathcal M$ under a policy $M$ satisfies an LTL formula $\phi$.  
A path $r^{M}_{\mathcal M}=s_{0}s_{1}\ldots$ deterministically generates a word $o=o_{0}o_{1}\ldots$,
where $o_{i}=\mathcal{L}(s_{i})$ for all $i$. With a slight abuse of notation, we denote $\mathcal{L}(r^{M}_{\mathcal M})$ as the word generated by $r^{M}_{\mathcal M}$.  Given an LTL formula $\phi$, one can show that the set $\{r^{M}_{\mathcal M} \in \mathrm{Paths}^{M}_{\mathcal M} : \mathcal{L}(r^{M}_{\mathcal M}) \vDash \phi\}$ is measurable.  We define:
\be
\label{eq:probofformula}
\textrm{Pr}^{M}_{\mathcal M}(\phi):=\textrm{Pr}^{M}_{\mathcal M}\{r^{M}_{\mathcal M} \in \mathrm{Paths}^{M}_{\mathcal M} : \mathcal{L}(r^{M}_{\mathcal M}) \vDash \phi\} \ee as the probability of satisfying $\phi$ for $\mathcal M$ under $M$.  For more details about probability measures on MDPs under a policy and measurability of LTL formulas, we refer readers to a text in probabilistic model checking, such as \cite{baier2008principles}.  

\section{Problem Formulation}
\label{sec:problem_formulation}

Consider a weighted MDP $\mathcal M=(S, U, P, s_{0}, \Pi, \mathcal L,
g)$ and an LTL formula $\phi$ over $\Pi$.  As proposed
in~\cite{SLS-JT-CB-DR:10b}, we assume that formula $\phi$ is of the
form: \be \phi=\gl \ev \pi \andltl \psi, \ee where the atomic
proposition $\pi\in \Pi$ is called the \emph{optimizing proposition}
and $\psi$ is an arbitrary LTL formula. In other words, $\phi$
requires that $\psi$ be satisfied and $\pi$ be satisfied infinitely
often.
We assume that there exists at least one policy $M$ of $\mathcal M$ such that 
$\mathcal M$ under $M$ satisfies $\phi$ almost surely, \ie $\textrm{Pr}^{M}_{\mathcal M}(\phi)=1$ (in this case we simply say $M$ satisfies $\phi$ almost surely).

We let $\mathbb M$ be the set of all policies and $\mathbb M_{\phi}$
be the set of all policies satisfying $\phi$ almost surely. Note that
if there exists a control policy satisfying $\phi$ almost surely, then
there typically exist many (possibly infinite number of) such policies.



We would like to obtain the optimal policy such that $\phi$ is almost
surely satisfied, and the expected cost in between visiting a state
satisfying $\pi$ is minimized.  To formalize this, we first denote
$S_{\pi}=\{i\in S, \pi\in \mathcal L(i)\}$ (\ie the states where
atomic proposition $\pi$ is true).  We say that each visit to set
$S_{\pi}$ {\it completes a cycle}.  Thus, starting at the initial
state, the finite path reaching $S_{\pi}$ for the first time is the
first cycle; the finite path that starts after the completion of the
first cycle and ends with revisiting $S_{\pi}$ for the second time is
the second cycle, and so on. Given a path $r^{M}_{\mathcal
  M}=s_{0}s_{1}\ldots$, we use $C(r^{M}_{\mathcal M},N)$ to denote the
cycle index up to stage $N$, which is defined as the total number of
cycles completed in $N$ stages plus 1 (\ie the cycle index starts with
$1$ at the initial state).


The main problem that we consider in this paper is to find a policy that minimizes the average cost per cycle (ACPC) starting from the initial state $s_{0}$. Formally, we have:

\begin{problem}
\label{prob:mainprob}
Find a policy $M=\{\mu_{0},\mu_{1},\ldots\}$, $M\in \mathbb M_{\phi}$
that minimizes
\be
\label{eq:averagecostpercycle}
J(s_{0})=\limsup_{N\rightarrow\infty} E\left\{\frac{\sum_{k=0}^{N}g(s_k,\mu_{k}(s_k))}{C(r^{M}_{\mathcal M},N)}\right\},
\ee
where $E\{\cdot\}$ denotes the expected value. 
\end{problem}


Prob. \ref{prob:mainprob} is related to the standard average cost per stage (ACPS) problem, which consist of minimizing 
\be
\label{eq:averagecost}
J^{s}(s_{0})=\limsup_{N\rightarrow\infty} E\left\{\frac{\sum_{k=0}^{N}g(s_k,\mu_{k}(s_k))}{N}\right\},
\ee
over $\mathbb M$, with the noted difference that the right-hand-side (RHS) of \eqref{eq:averagecost} is divided by the index of stages instead of cycles.  The ACPS problem has been widely studied in the dynamic programming community, without the constraint of satisfying temporal logic formulas.  

The ACPC cost function we consider in this paper is relevant for 
probabilistic abstractions and practical applications, where the cost of controls can represent the time, energy, or fuel required to apply controls at each state.  In particular, it is a suitable performance measure for persistent tasks, which can be specified by LTL formulas.  For example, in a data gathering mission \cite{SLS-JT-CB-DR:10b}, an agent is required to repeatedly gather and upload data.   We can assign $\pi$ to the data upload locations and a solution to Prob. \ref{prob:mainprob} minimizes the expected cost in between  data upload.  In such cases, the ACPS cost function does not translate to a meaningful performance criterion. 
In fact, a policy minimizing \eqref{eq:averagecost} may even produce an infinite cost in \eqref{eq:averagecostpercycle}.
Nevertheless, we will make the connection between the ACPS and the ACPC problems in Sec. \ref{sec:solutiontoaveragecycle}. 

\begin{remark}[Optimization Criterion]
The optimization criterion in Prob. \ref{prob:mainprob} is only meaningful for specifications where $\pi$ is satisfied infinitely often.  Otherwise, the limit from Eq. (\ref{eq:averagecostpercycle}) is infinite (since $g$ is a positive-valued function) and Prob. \ref{prob:mainprob} has no solution.  This is the reason for choosing $\phi$ in the form $\gl \ev \pi \andltl \psi$ and for only searching among policies that almost surely satisfy $\phi$.  
\end{remark}


\section{Solving the average cost per cycle problem}
\label{sec:solutiontoaveragecycle}
\subsection{Optimality conditions for ACPS problems}
\label{sec:sub:averagecostprelim}
In this section, we recall some known results on the ACPS problem, namely finding a policy over $\mathbb M$ that minimizes $J^{s}$ in \eqref{eq:averagecost}. The reader interested in more details is referred to   
\cite{bertsekas2007dynamic,puterman1994markov} and references therein.

\begin{definition}[Weak Accessibility Condition]
\label{def:WA}
An MDP $\mathcal M$ is said to satisfy the Weak Accessibility (WA) condition if there exist $S_{r}\subseteq S$, such that (i) there exists a stationary policy where $j$ is reachable from $i$ for any $i,j\in S_{r}$, and (ii) states in $S\setminus S_{r}$ are transient under all stationary policies.
\end{definition}

MDP $\mathcal M$ is called \emph{single-chain} (or \emph{weakly-communicating}) if it satisfies the WA condition. If $\mathcal M$ satisfies the WA condition with $S_{r}=S$, then $\mathcal M$ is called \emph{communicating}.  For a stationary policy, it induces a Markov chain with a set of recurrent classes.  A state that does not belong to any recurrent class is called \emph{transient}.  A stationary policy $\mu$ is called \emph{unichain} if the Markov chain induced by $\mu$ contains one recurrent class (and a possible set of transient states). If every stationary policy is unichain, $\mathcal M$ is called unichain.

Recall that the set of states of $\mathcal M$ is denoted by $\{1,\ldots,n\}$.  For each stationary policy $\mu$, we 
use $P_{\mu}$ to denote the transition probability matrix: $P_{\mu}(i,j)=P(i,\mu(i),j)$.  
Define vector $g_{\mu}$ where $g_{\mu}(i)=g(i,\mu(i))$.  
%
For each stationary policy $\mu$, we can obtain a so-called gain-bias pair $(J^{s}_{\mu},h^{s}_{\mu})$, where 
\be
\label{eq:gainbais}
J^{s}_{\mu}=P^{\ast}_{\mu}g_{\mu}, \hspace{.5cm} h^{s}_{\mu}=H^{s}_{\mu}g_{\mu}
\ee
with
\be
\label{eq:pmustar}
P_{\mu}^{\ast}=\lim_{N\rightarrow\infty}\frac{1}{N}\sum_{k=0}^{N-1}P_{\mu}^{k}, \hspace{.5cm} H_{\mu}=(I-P_{\mu}+P_{\mu}^{\ast})^{-1}-P_{\mu}^{\ast}.
\ee
The vector $J^{s}_{\mu}=[J^{s}_{\mu}(1),\ldots,J^{s}_{\mu}(n)]^{\tt T}$ is such that $J^{s}_{\mu}(i)$ is the ACPS starting at initial state $i$ under policy $\mu$.  
Note that the limit in \eqref{eq:pmustar} exists for any stochastic matrix $P_{\mu}$, and $P_{\mu}^{\ast}$ is stochastic. Therefore, the $\limsup$ in \eqref{eq:averagecost} can be replaced by the limit for a stationary policy.  Moreover, $(J^{s}_{\mu},h^{s}_{\mu})$ satisfies
\be
\label{eq:helpereq}
J^{s}_{\mu}=P_{\mu}J^{s}_{\mu}, \hspace{.5cm} J^{s}_{\mu}+h^{s}_{\mu}=g_{\mu}+P_{\mu}h^{s}_{\mu}.
\ee
By noting that
\be
\label{eq:vmu}
h^{s}_{\mu}+v^{s}_{\mu}=P_{\mu}v^{s}_{\mu}, %
\ee %
for some vector $v^{s}_{\mu}$, we see that
$(J^{s}_{\mu},h^{s}_{\mu},v^{s}_{\mu})$ is the solution of $3n$ linear
equations with $3n$ unknowns.

It has been shown that there exists a stationary optimal policy $\mu^{\star}$ minimizing \eqref{eq:averagecost} over all policies, where its gain-bias pair $(J^{s},h^{s})$ satisfies the Bellman's equations for average cost per stage problems:
\be
\label{eq:bellmanacps1}
J^{s}(i)=\min_{u\in U(i)}\sum_{j=1}^{n}P(i,u,j)J^{s}(j)
\ee
and 
\be
\label{eq:bellmanacps2}
J^{s}(i)+h^{s}(i)=\min_{u\in \bar U(i)}\bigg[g(i,u)+\sum_{j=1}^{n}P(i,u,j)h^{s}(j)\bigg],
\ee
for all $i=1,\ldots, n$, where $\bar{U}_{i}$ is the set of controls attaining the minimum in \eqref{eq:bellmanacps1}.  Furthermore, if $\mathcal M$ is single-chain, the optimal average cost does not depend on the initial state, \ie $J^{s}_{\mu^{\star}}(i)=\lambda$ for all $i\in S$.  In this case, \eqref{eq:bellmanacps1} is trivially satisfied and $\bar U_{i}$ in \eqref{eq:bellmanacps2} can be replaced by $U(i)$.  Hence,  $\mu^{\star}$ with gain-bias pair $(\lambda\bo, h)$ is optimal over all polices if for all stationary policies $\mu$ we have:
\be
\label{eq:bellmaninmatrixmu}
\lambda\bo+h\leq g_{\mu}+P_{\mu}h,
\ee 
where $\bo\in \mathbb R^{n}$ is a vector of all $1$s and $\leq$ is component-wise.

\subsection{Optimality conditions for ACPC problems}
\label{sec:sub:optimalACPC}
Now we derive equations similar to \eqref{eq:bellmanacps1} and \eqref{eq:bellmanacps2} for ACPC problems, without considering the satisfaction constraint, \ie we do not limit the set of polices to $\mathbb M_{\phi}$ at the moment. We consider the following problem:%
\begin{problem}
\label{prob:averagecostpercycle}
Given a communicating MDP $\mathcal M$ and a set $S_{\pi}$, find a policy $\mu\in\mathbb M$ that minimizes \eqref{eq:averagecostpercycle}.
\end{problem}

Note that, for reasons that will become clear in Sec. \ref{sec:solveLTLprob}, we assume in Prob. \ref{prob:averagecostpercycle} that the MDP is communicating.  However, it is possible to generalize the results in this section to an MDP that is not communicating.

We limit our attention to stationary policies. We will show that,  similar to the majority of problems in dynamic programming, there exist optimal stationary policies, thus it is sufficient to consider only stationary policies.  For such policies, we use the following notion of \emph{proper policies}, which is the same as the one used in stochastic shortest path problems (see \cite{bertsekas2007dynamic}).%
\begin{definition}[Proper Polices]
We say a stationary policy $\mu$ is proper if, under $\mu$, all initial states 
have positive probability to reach the set $S_{\pi}$ in a finite number of stages. 
\end{definition}

We denote $J_{\mu}=[J_{\mu}(1),\ldots,J_{\mu}(n)]^{\tt T}$ where $J_{\mu}(i)$ is the ACPC in \eqref{eq:averagecostpercycle} starting from state $i$ under policy $\mu$.  
If policy $\mu$ is improper, then there exist some states $i\in S$ that can never reach $S_{\pi}$.  In this case, since $g(i,u)$ is positive for all $i,u$, we can immediately see that $J_{\mu}(i)=\infty$.  
We will first consider only proper policies.  

Without loss of generality, we assume that $S_{\pi}=\{1,\ldots,m\}$ (\ie states $m+1,\ldots, n$ are not in $S_{\pi}$).   Given a proper policy $\mu$, we obtain its transition matrix $P_{\mu}$ as described in Sec. \ref{sec:sub:averagecostprelim}.  Our goal is to express $J_{\mu}$ in terms of $P_{\mu}$, similar to \eqref{eq:gainbais} in the ACPS case.
To achieve this, we first compute the probability that $j\in S_{\pi}$ is the first state visited in $S_{\pi}$ after leaving from a state $i\in S$ by applying policy $\mu$.  We denote this probability by $\widetilde{P}(i,\mu,j)$.  We can obtain this probability for all $i\in S$ and $j\in S_{\pi}$ by the following proposition: 

\begin{proposition}
$\widetilde P(i,\mu,j)$ satisfies
\be
\label{eq:tildePij}
\widetilde P(i,\mu,j)=\sum_{k=m+1}^{n} P(i,\mu(i),k)\widetilde P(k,\mu(k),j) + P(i,\mu(i),j).
\ee
\end{proposition}
\begin{proof}
From $i$, the next state can either be in $S_{\pi}$ or not.  The first term in the RHS of \eqref{eq:tildePij} is the probability of reaching $S_{\pi}$ and the first state is $j$, given that the next state is not in $S_{\pi}$.  Adding it with the probability of next step is in $S_{\pi}$ and the state is $j$ gives the desired result.
\end{proof}

We now define a $n\times n$ matrix $\widetilde{P}_{\mu}$ such that
\be
 \widetilde{P}_{\mu}(i,j)=\left\{
\begin{array}{ll}
\widetilde P(i,\mu,j) & \text{if } j\in S_{\pi}\\ 0 & \text{otherwise}
\end{array}\right.
\ee
We can immediately see that  $\widetilde{P}_{\mu}$ is a stochastic matrix, \ie all its rows sum up to $1$ or $\sum_{j=1}^{n}\widetilde P(i,\mu,j)=1$.  More precisely, $\sum_{j=1}^{m}\widetilde P(i,\mu,j)=1$ since $\widetilde P(i,\mu,j)=0$ for all $j=m+1,\ldots,n$.

Using \eqref{eq:tildePij}, we can express $\widetilde{P}_{\mu}$ in a matrix equation in terms of $P_{\mu}$.  To do this, we need to first define two $n\times n$ matrices from $P_{\mu}$ as follows:
\be
 \overleftarrow{P}_{\mu}(i,j)=\left\{
\begin{array}{ll}
P_{\mu}(i,j) & \text{if } j\in S_{\pi} \\ 0 & \text{otherwise}
\end{array}\right.
\ee
\be
 \overrightarrow{P}_{\mu}(i,j)=\left\{
\begin{array}{ll}
P_{\mu}(i,j) & \text{if } j\notin S_{\pi} \\ 0 & \text{otherwise}
\end{array}\right.
\ee
From Fig. \ref{fig:splitPmu}, we can see that matrix $P_{\mu}$ is ``split'' into $\overleftarrow{P}_{\mu}$ and $\overrightarrow{P}_{\mu}$, \ie $P_{\mu}=\overleftarrow{P}_{\mu}+\overrightarrow{P}_{\mu}$.

\begin{figure}[h]
\begin{center}
\includegraphics[scale=0.3]{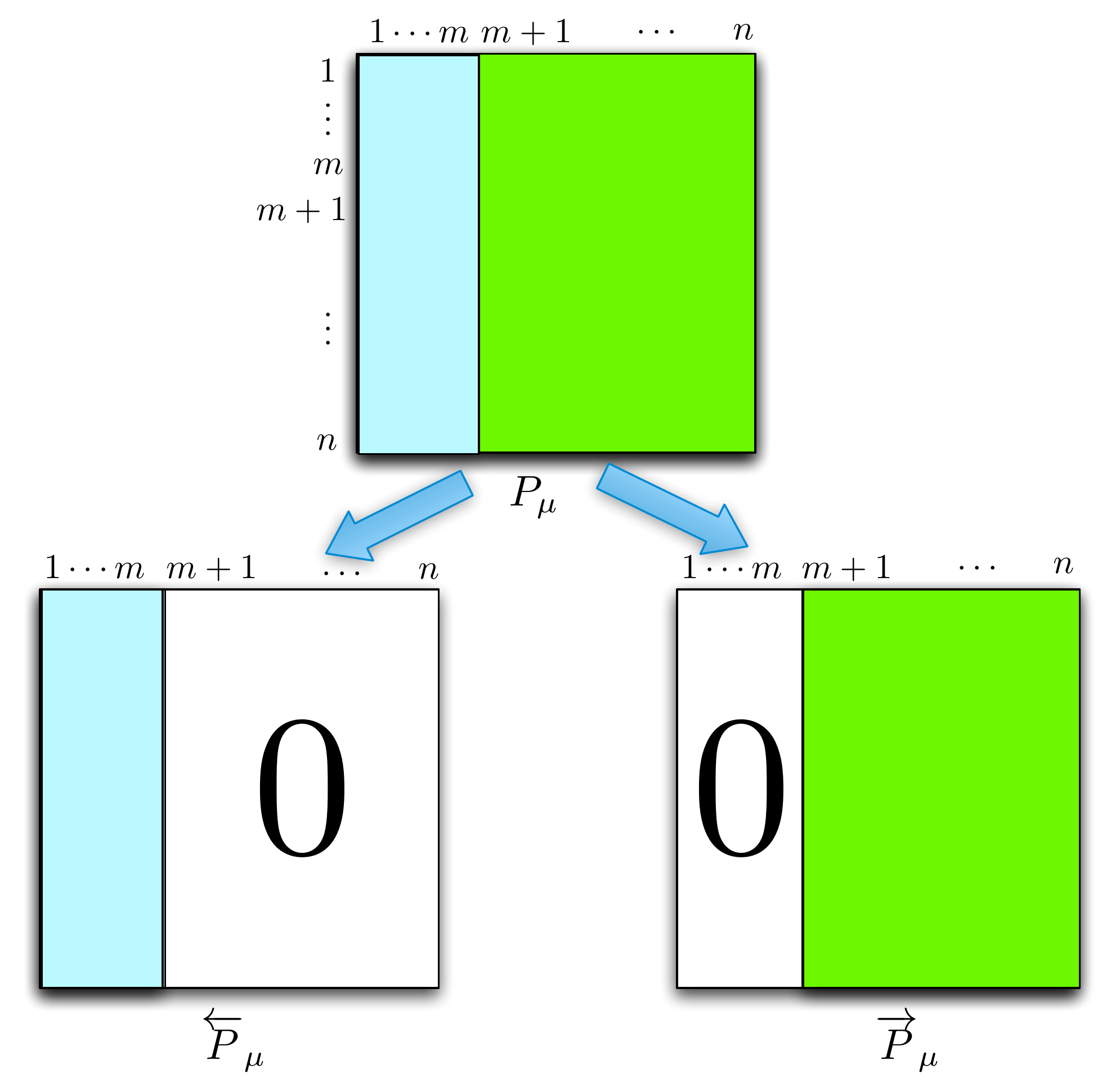}
\caption{The constructions of  $\overset{\leftarrow}{P}_{\mu}$ and $\overset{\rightarrow}{P}_{\mu}$  from $P_{\mu}$.}
\label{fig:splitPmu}
\end{center}
\end{figure}


\begin{proposition}
\label{prop:matrixI-PmuInvertible}
If a policy $\mu$ is proper, then matrix $I-\overrightarrow{P}_{\mu}$ is non-singular.
\end{proposition}
\begin{proof}
Since $\mu$ is proper, for every initial state $i\in S$, the set $S_{\pi}$ is eventually reached.  Because of this, and since $\overrightarrow P_{\mu}(i,j)=0$ if $j\in S_{\pi}$, matrix $\overrightarrow P_{\mu}$ is transient, \ie $\lim_{k\rightarrow\infty}\overrightarrow P_{\mu}^{k}=0$.  From linear algebra (see, \eg Ch. $9.4$ of \cite{hogben2007handbook}), since $\overrightarrow{P}_{\mu}$ is transient and sub-stochastic, $I-\overrightarrow{P}_{\mu}$ is non-singular.
\end{proof}

We can then write \eqref{eq:tildePij} as the following matrix equation:
\be
\label{eq:pmumatrix}
\widetilde{P}_{\mu}=\overrightarrow{P}_{\mu}\widetilde{P}_{\mu}+\overleftarrow{P}_{\mu}.
\ee
Since $I-\overrightarrow{P}_{\mu}$ is invertible, we have
\be
\label{eq:pmumatrixform}
\widetilde{P}_{\mu}=(I-\overrightarrow{P}_{\mu})^{-1}\overleftarrow{P}_{\mu}.
\ee
Note that \eqref{eq:pmumatrix} and \eqref{eq:pmumatrixform} do not depend on the ordering of the states of $\mathcal M$, \ie $S_{\pi}$ does not need to be equal to $\{1,\ldots,m\}$.

Next, we give an expression for the expected cost of reaching $S_{\pi}$ from $i\in S$ under $\mu$ (if $i\in S_{\pi}$, this is the expected cost of reaching $S_{\pi}$ again), and denote it as $\tilde g(i,\mu)$.  
\begin{proposition}
$\tilde g(i,\mu)$ satisfies
\be
\label{eq:gmu}
\tilde g(i,\mu)=\sum_{k=m+1}^{n} P(i,\mu(i),k)\tilde g(k,\mu)+g(i,\mu(i)).
\ee
\end{proposition}
\begin{proof}
The first term of RHS of \eqref{eq:gmu} is the expected cost from the next state if the next state is not in $S_{\pi}$ (if the next state is in $S_{\pi}$ then no extra cost is incurred), and the second term is the one-step cost, which is incurred regardless of the next state.
\end{proof}

We define $\tilde g_{\mu}$ such that $\tilde g_{\mu}(i)=\tilde g(i,\mu)$, and note that \eqref{eq:gmu} can be written as:
\bea
\label{eq:gmumatrixform}
\tilde g_{\mu}&=&\overrightarrow P_{\mu}\tilde g_{\mu}+g_{\mu}\n
\tilde g_{\mu}&=&(I-\overrightarrow{P}_{\mu})^{-1}g_{\mu},
\eea
where $g_{\mu}$ is defined in Sec. \ref{sec:sub:averagecostprelim}.

We can now express the ACPC $J_{\mu}$ in terms of $\widetilde{P}_{\mu}$ and $\tilde g_{\mu}$.  Observe that, starting from $i$, the expected cost of the first cycle is $\tilde g_{\mu}(i)$; the expected cost of the second cycle is $\sum_{j=1}^{m}\widetilde P_{\mu}(i,\mu,j)\tilde g_{\mu}(j)$; the expected cost of the third cycle is $\sum_{j=1}^{m}\sum_{k=1}^{m}\widetilde P_{\mu}(i,\mu,j)\widetilde P_{\mu}(j,\mu,k)\tilde g_{\mu}(k)$; and so on.  Therefore:
\be
\label{eq:averagecostallcycles}
J_{\mu}=\limsup_{C\rightarrow\infty} \frac{1}{C}\sum_{k=0}^{C-1}\widetilde{P}_{\mu}^{k}\tilde g_{\mu},
\ee
where $C$ represents the cycle count.  Since $\widetilde{P}_{\mu}$ is a stochastic matrix, the $\limsup$ in \eqref{eq:averagecostallcycles} can be replaced by the limit, and we have
\be
\label{eq:costinmatrix}
J_{\mu}=\lim_{C\rightarrow\infty} \frac{1}{C}\sum_{k=0}^{C-1}\widetilde{P}_{\mu}^{k}\tilde g_{\mu}=\widetilde{P}_{\mu}^{\ast}\tilde g_{\mu},
\ee
where $P^{\ast}$ for a stochastic matrix $P$ is defined in \eqref{eq:pmustar}.  

We can now make a connection between Prob. \ref{prob:averagecostpercycle} and the ACPS problem. Each proper policy $\mu$ of $\mathcal M$ can be mapped to a policy $\tilde \mu$ with transition matrix $P_{\tilde\mu}:=\widetilde{P}_{\mu}$ and vector of costs $g_{\tilde\mu}:=\tilde g_{\mu}$, and we have
\be
\label{eq:costequiv}
J_{\mu}=J^{s}_{\tilde\mu}.
\ee
Moreover, we define $h_{\mu}:=h^{s}_{\tilde\mu}$.  Together with $J_{\mu}$, pair $(J_{\mu},h_{\mu})$ can be seen as the gain-bias pair for the ACPC problem.   
We denote the set of all polices that can be mapped to a proper policy as $\mathbb M_{\tilde\mu}$.  We see that a proper policy minimizing the ACPC maps to a policy over $\mathbb M_{\tilde\mu}$ minimizing the ACPS.

The by-product of the above analysis is that, if $\mu$ is proper, then $J_{\mu}(i)$ is finite for all $i$, since $\widetilde P_{\mu}^{\ast}$ is a stochastic matrix and $g_{\mu}(i)$ is finite.  We now show that, among stationary policies, it is sufficient to consider only proper policies. 
\begin{proposition}
\label{prop:propersufficient}
Assume $\mu$ to be an improper policy.  If $\mathcal M$ is communicating, then there exists a proper policy $\mu'$ such that $J_{\mu'}(i)\leq J_{\mu}(i)$ for all $i=1,\ldots,n$, with strict inequality for at least one $i$.
\end{proposition}
\begin{proof}
We partition $S$ into two sets of states: $S_{\nrightarrow \pi}$ is the set of states in $S$ that cannot reach $S_{\pi}$ and $S_{\rightarrow \pi}$ as the set of states that can reach $S_{\pi}$ with positive probability.  Since $\mu$ is improper and $g(i,u)$ is postive-valued, $S_{\nrightarrow \pi}$ is not empty and $J_{\mu}(i)=\infty$ for all $i\in S_{\nrightarrow \pi}$.   Moreover, states in $S_{\nrightarrow \pi}$ cannot visit $S_{\rightarrow \pi}$ by definition.  Since $\mathcal M$ is communicating, there exists some actions at some states in $S_{\nrightarrow \pi}$ such that, if applied, all states in $S_{\nrightarrow \pi}$ can now visit $S_{\pi}$ with positive probability and this policy is now proper (all states can now reach $S_{\pi}$).   We denote this new policy as $\mu'$.  Note that this does not increase $J_{\mu}(i)$ if  $i\in S_{\rightarrow\pi}$ since controls at these states are not changed.  Moreover, since $\mu'$ is proper, $J_{\mu'}(i)< \infty$ for all $i\in S_{\nrightarrow \pi}$.  Therefore $J_{\mu'}(i)< J_{\mu}(i)$ for all $i\in S_{\nrightarrow \pi}$.
\end{proof}

Using the connection to the ACPS problem, we have:
\begin{proposition}
\label{prop:optimalcostsame}
The optimal ACPC policy over stationary policies is independent of the initial state.
\end{proposition}
\begin{proof}
We first consider the optimal ACPC over proper policies.  As mentioned before, if all stationary policies of an MDP satisfies the WA condition (see Def. \ref{def:WA}), then the ACPS is equal for all initial states.  Thus, we need to show that the WA condition is satisfied for all $\tilde\mu$.  We will use $S_{\pi}$ as set $S_{r}$.  Since $\mathcal M$ is communicating, then for each pair $i,j\in S_{\pi}$, $P(i,\mu,j)$ is positive for some $\mu$, therefore from \eqref{eq:tildePij}, $\widetilde P_{\mu}(i,j)$ is positive for some $\mu$ (\ie $P_{\tilde\mu}(i,j)$ is positive for some $\tilde\mu$), and the first condition of Def. \ref{def:WA} is satisfied.  Since $\mu$ is proper, the set $S_{\pi}$ can be reached from all $i\in S$.  In addition, $P_{\tilde\mu}(i,j)=0$ for all $j\notin S_{\pi}$.  Thus, all states $i\notin S_{\pi}$ are transient under all policies $\tilde \mu\in \mathbb M_{\tilde\mu}$, and the second condition is satisfied.  Therefore WA condition is satisfied and the optimal ACPS over $\mathbb M_{\tilde\mu}$ is equal for all initial state.  Hence, the optimal ACPC is the same for all initial states over proper policies.  Using Prop. \ref{prop:propersufficient}, we can conclude the same statement over stationary policies.
\end{proof}

The above result is not surprising, as it mirrors the result for a single-chain MDP in the ACPS problem.  Essentially, transient behavior does not matter in the long run so the optimal cost is the same for any initial state.  

Using Bellman's equation \eqref{eq:bellmanacps1} and \eqref{eq:bellmanacps2}, and in particular the case when the optimal cost is the same for all initial states \eqref{eq:bellmaninmatrixmu}, policy $\tilde\mu^{\star}$ with the ACPS gain-bias pair $(\lambda\bo,h)$ satisfying for all $\tilde\mu\in\mathbb M_{\tilde\mu}$:
\be
\label{eq:bellmaninmatrix}
\lambda\bo+h\leq g_{\tilde \mu}+P_{\tilde\mu}h
\ee
is optimal.  Equivalently, $\mu^{\star}$ that maps to $\tilde\mu^{\star}$ is optimal over all proper policies. The following proposition shows that this policy is optimal over all policies in $\mathbb M$, stationary or not.
\begin{proposition}
\label{prop:optimalpolicyIsStationary}
The proper policy $\mu^{\star}$ that maps to $\tilde\mu^{\star}$ satisfying \eqref{eq:bellmaninmatrix} is optimal over $\mathbb M$.
\end{proposition}
\begin{proof}
Consider a $M=\{\mu_{1},\mu_{2,}\ldots\}$ and assume it to be optimal.  We first consider that $M$ is stationary for each cycle, and the policy is $\mu_{k}$ for the $k$-th cycle.  Among this type of polices, from Prop. \ref{prop:propersufficient}, we see that if $M$ is optimal, then $\mu_{k}$ is proper for all $k$.  In addition, the ACPC of policy M is the ACPS with policy $\{\tilde\mu_{1},\tilde\mu_{2},\ldots\}$.   Since the optimal policy of the ACPS is $\tilde\mu^{\star}$ (stationary).  Then we can conclude that if $M$ is stationary in between cycles, then optimal policy for each cycle is $\mu^{\star}$ and thus $M=\mu^{\star}$.   

Now we assume that $M$ is not stationary for each cycle.  Since $g(i,u)>0$ for all $i,u$, and there exists at least one proper policy, the stochastic shortest path problem for $S_{\pi}$ admits an optimal stationary policy as a solution \cite{bertsekas2007dynamic}.    Hence, for each cycle $k$, the cycle cost can be minimized if a stationary policy is used for the cycle.  Therefore, a policy which is stationary in between cycles is optimal over $\mathbb M$, which is in turn, optimal if $M=\mu^{\star}$.  The proof is complete. 
\end{proof}

Unfortunately, it is not clear how to find the optimal policy from \eqref{eq:bellmaninmatrix} except by searching through all policies in $\mathbb M_{\tilde\mu}$.  This exhaustive search is not feasible for reasonably large problems.  Instead, we would like equations in the form of \eqref{eq:bellmanacps1} and \eqref{eq:bellmanacps2}, so that the optimizations can be carried out independently at each state.


To circumvent this difficulty, we need to express the gain-bias pair $(J_{\mu},h_{\mu})$, which is equal to $(J^{s}_{\tilde\mu},h^{s}_{\tilde\mu})$, in terms of $\mu$.  From \eqref{eq:helpereq}, we have 
\ben
J_{\mu}=P_{\tilde\mu}J_{\mu}, \hspace{.5cm} J_{\mu}+h_{\mu}=g_{\tilde\mu}+P_{\tilde\mu}h_{\mu}.
\een
By manipulating the above equations, we can now show that $J_{\mu}$ and $h_{\mu}$ can be expressed in terms of $\mu$ (analogous to \eqref{eq:helpereq}) instead of $\tilde\mu$ via the following proposition:
\begin{proposition}
\label{prop:Jmuhmu}
We have
\be
\label{eq:helpereqmu}
J_{\mu}=P_{\mu}J_{\mu},  \hspace{.5cm} J_{\mu}+h_{\mu}=g_{\mu}+P_{\mu}h_{\mu}+\overrightarrow P_{\mu}J_{\mu}.
\ee
Moreover, we have
\be
\label{eq:vmutilde}
(I-\overrightarrow P_{\mu})h_{\mu}+v_{\mu}= P_{\mu}v_{\mu},
\ee
for some vector $v_{\mu}$.
\end{proposition}
\begin{proof}
We start from \eqref{eq:helpereq}:
\be
\label{eq:helpereqtilde}
J_{\mu}=P_{\tilde\mu}J_{\mu}, \hspace{.5cm} J_{\mu}+h_{\mu}=g_{\tilde\mu}+P_{\tilde\mu}h_{\mu}.
\ee
For the first equation in \eqref{eq:helpereqtilde}, using \eqref{eq:pmumatrixform}, we have
\bean
\label{eq:jmusimplify}
J_{\mu}&=&P_{\tilde\mu}J_{\mu}\\
J_{\mu}&=&(I-\overrightarrow P_{\mu})^{-1}\overleftarrow P_{\mu}J_{\mu}\\
(I-\overrightarrow P_{\mu})J_{\mu}&=&\overleftarrow P_{\mu}J_{\mu}\\
J_{\mu}-\overrightarrow P_{\mu}J_{\mu}&=&\overleftarrow P_{\mu}J_{\mu}\\
J_{\mu}&=&(\overrightarrow P_{\mu}+\overleftarrow P_{\mu})J_{\mu}\\
J_{\mu}&=&P_{\mu}J_{\mu}.
\eean
For the second equation in \eqref{eq:helpereqtilde}, using \eqref{eq:pmumatrixform} and \eqref{eq:gmumatrixform}, we have
\bean
\label{eq:hmusimplify}
J_{\mu}+h_{\mu}&=&g_{\tilde\mu}+P_{\tilde\mu}h_{\mu}\\
J_{\mu}+h_{\mu}&=&(I-\overrightarrow P_{\mu})^{-1}(g_{\mu}+\overleftarrow P_{\mu}h_{\mu})\\
(I-\overrightarrow P_{\mu})(J_{\mu}+h_{\mu})&=&g_{\mu}+\overleftarrow P_{\mu}h_{\mu}\\
J_{\mu}-\overrightarrow P_{\mu} J_{\mu}+h_{\mu}-\overrightarrow P_{\mu}h_{\mu}&=&g_{\mu}+\overleftarrow P_{\mu}h_{\mu}\\
J_{\mu}+h_{\mu}-\overrightarrow P_{\mu}J_{\mu}&=&g_{\mu}+(\overrightarrow P_{\mu}+\overleftarrow P_{\mu})h_{\mu}\\
J_{\mu}+h_{\mu}&=&g_{\mu}+P_{\mu}h_{\mu}+\overrightarrow P_{\mu}J_{\mu}.
\eean
Thus, \eqref{eq:helpereqtilde} can be expressed in terms of $\mu$ as:
\ben
J_{\mu}=P_{\mu}J_{\mu},  \hspace{.5cm} J_{\mu}+h_{\mu}=g_{\mu}+P_{\mu}h_{\mu}+\overrightarrow P_{\mu}J_{\mu}.
\een
To compute $J_{\mu}$ and $h_{\mu}$, we need an extra equation similar to \eqref{eq:vmu}.  Using \eqref{eq:vmu}, we have
\bean
h_{\mu}+v_{\mu}&=&P_{\tilde\mu}v_{\mu}\n
h_{\mu}+v_{\mu}&=&(I-\overrightarrow P_{\mu})^{-1}\overleftarrow P_{\mu}v_{\mu}\n
(I-\overrightarrow P_{\mu})h_{\mu}+v_{\mu}&=& P_{\mu}v_{\mu},
\eean
which completes the proof.
\end{proof}

From Prop. \ref{prop:Jmuhmu}, similar to the ACPS problem, $(J_{\mu},h_{\mu},v_{\mu})$ can be solved together by a linear system of $3n$ equations and $3n$ unknowns.  The insight gained when simplifying $J_{\mu}$ and $h_{\mu}$ in terms of $\mu$ motivate us to propose the following optimality condition for an optimal policy.

\begin{figure*}[!ht]
   \centering
   \subfloat[]{\includegraphics[scale=0.35]{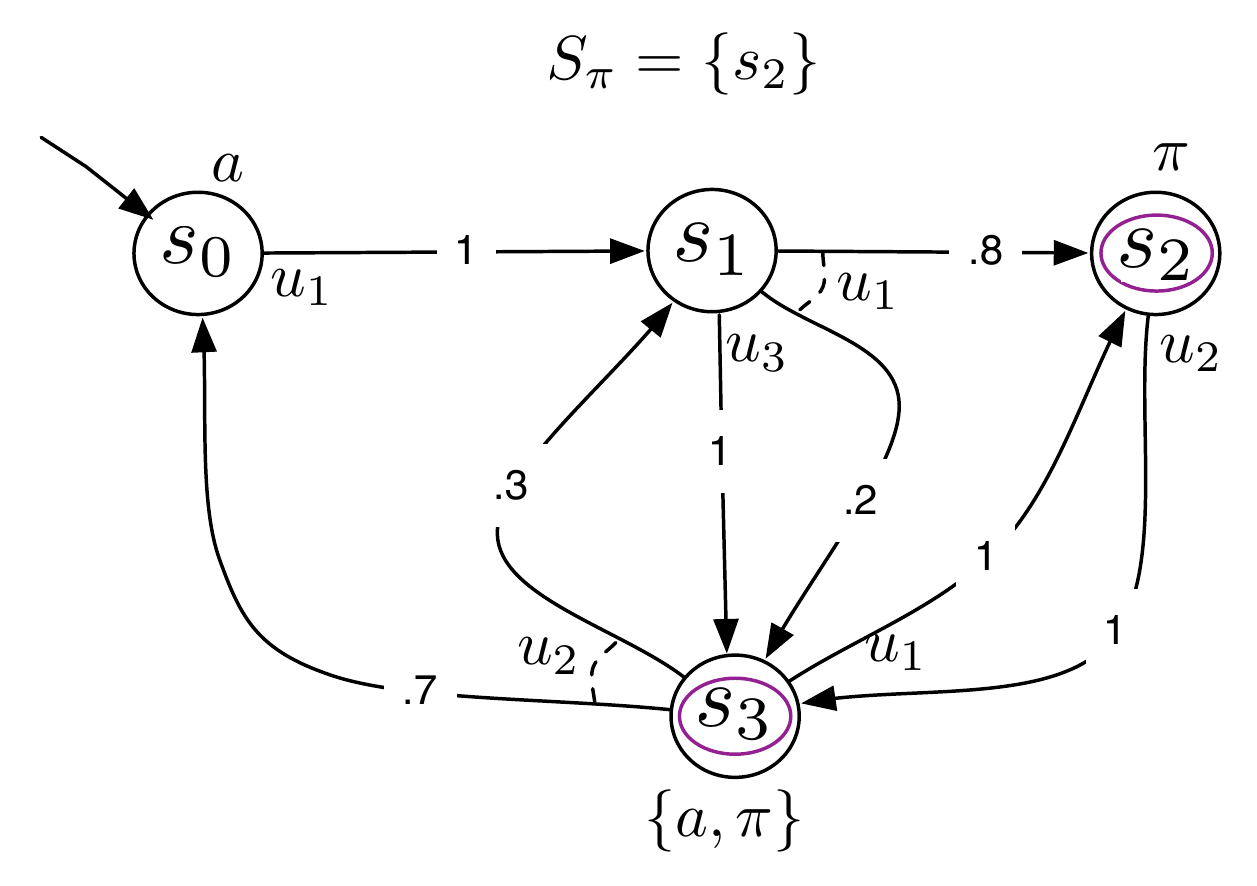}}
   \subfloat[]{\includegraphics[scale=0.35]{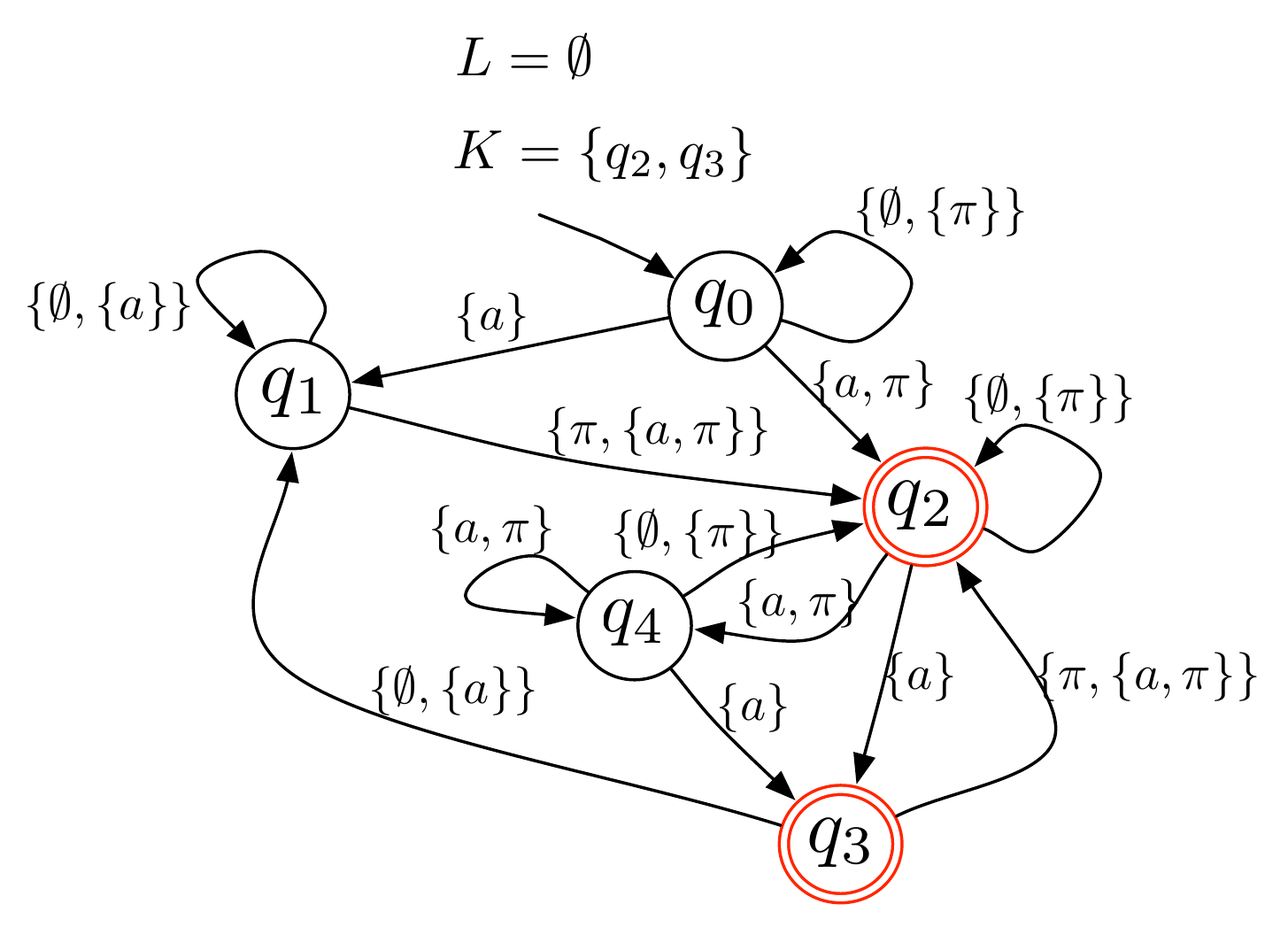}}
   \subfloat[]{\includegraphics[scale=0.35]{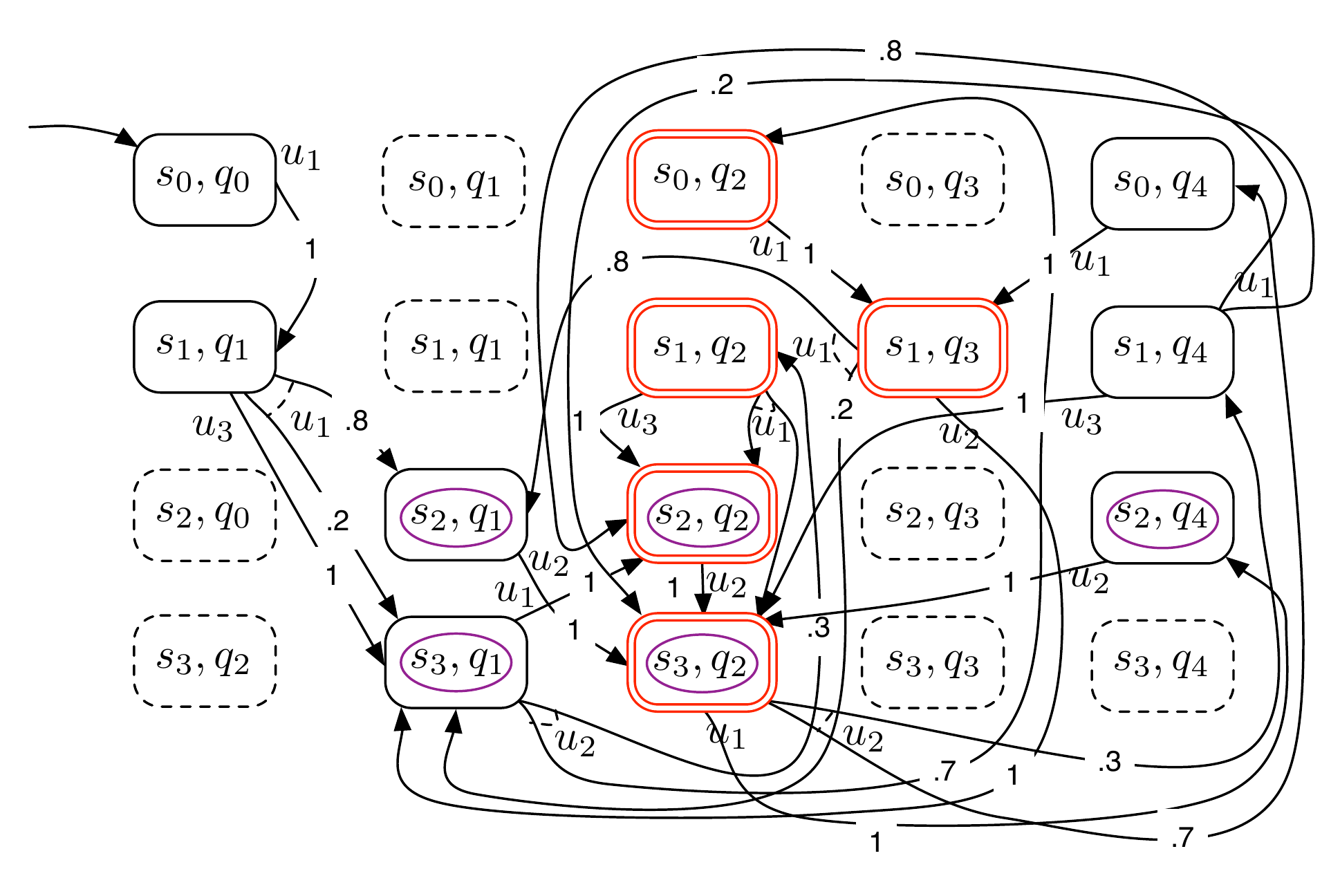}}
\caption{The construction of the product MDP between a labeled MDP and a DRA.  In this example, the set of atomic proposition is $\{a,\pi\}$.  {\bf (a)}: A labeled MDP where the label on top of a state denotes the atomic propositions assigned to the state.  The number on top of an arrow pointing from a state $s$ to $s'$ is the probability $P(s,u,s')$ associated with a control $u\in U(s)$.  The set of states marked by ovals is $S_{\pi}$.  {\bf (b)}: The DRA $\mathcal R_{\phi}$ corresponding to the LTL formula $\phi=\gl\ev \pi \andltl \gl\ev a$.  In this example, there is one set of accepting states $F=\{(L,K)\}$ where $L=\emptyset$ and $K=\{q_{2},q_{3}\}$ (marked by double-strokes).   Thus, accepting runs of this DRA must visit $q_{2}$ or $q_{3}$ (or both) infinitely often.  {\bf (c)}: The product MDP $\mathcal P=\mathcal M\times\mathcal R_{\phi}$ where states of $K_{\mathcal P}$ are marked by double-strokes and states of $S_{\mathcal P\pi}$ are marked by ovals.   The states with dashed borders are unreachable, and they are removed from $S_{\mathcal P}$.}
\label{fig:MDP-DRA-Product}
\end{figure*}

\begin{proposition}
\label{prop:optimalityconditiongeneral}
The policy $\mu^{\star}$ with gain-bias pair $(\lambda\bo,h)$ satisfying
\begin{multline}
\label{eq:bellmanpercycle}
\lambda+h(i)=\min_{u\in U(i)}\\\bigg[g(i,u)+\sum_{j=1}^{n} P(i,u,j)h(j)+\lambda\sum_{j=m+1}^{n} P(i,u,j)\bigg],
\end{multline}
for all $i=1,\ldots,n$, is the optimal policy minimizing \eqref{eq:averagecostpercycle} over all policies in $\mathbb M$.
\end{proposition}
\begin{proof}
The optimality condition \eqref{eq:bellmanpercycle} can be written as:
\be
\label{eq:bellmanmatrixformu}
\lambda\bo+h\leq g_{\mu}+P_{\mu}h+\overrightarrow P_{\mu}\lambda\bo,
\ee
for all stationary policies $\mu$.   

Note that, given $a,b\in \mathbb R^{n}$ and $a\leq b$, if $A$ is an $n\times n$ matrix with all non-negative entries, then $Aa\leq Ab$.  Moreover, given $c\in \mathbb R^{n}$, we have $a+c\leq b+c$. 

From \eqref{eq:bellmanmatrixformu} we have
\bea
\label{eq:bellmaninmatrixrerranged1}
\lambda\bo+h&\leq&g_{\mu}+P_{\mu}h+\overrightarrow{P}_{\mu}\lambda\bo\n
\lambda\bo-\overrightarrow{P}_{\mu}\lambda\bo+h&\leq&g_{\mu}+P_{\mu}h\n
\lambda\bo-\overrightarrow{P}_{\mu}\lambda\bo+h&\leq&g_{\mu}+(\overleftarrow{P}_{\mu}+\overrightarrow{P}_{\mu})h\n
\lambda\bo-\overrightarrow{P}_{\mu}\lambda\bo+h-\overrightarrow{P}_{\mu}h&\leq&g_{\mu}+\overleftarrow{P}_{\mu}h\n
 (I-\overrightarrow{P}_{\mu})(\lambda\bo+h)&\leq&g_{\mu}+\overleftarrow{P}_{\mu}h
\eea

If $\mu$ is proper, then $\overrightarrow{P}_{\mu}$ is a transient matrix (see proof of Prop. \ref{prop:matrixI-PmuInvertible}), and all of its eigenvalues are strictly inside the unit circle.  Therefore, we have
\ben
(I-\overrightarrow{P}_{\mu})^{-1}=I+\overrightarrow{P}_{\mu}+\overrightarrow{P}_{\mu}^{2}+\ldots.
\een
Therefore, since all entries of $\overrightarrow P_{\mu}$ are non-negative, all entries of $(I-\overrightarrow{P}_{\mu})^{-1}$ are non-negative.  Thus, continuing from \eqref{eq:bellmaninmatrixrerranged1}, we have
\bea
(I-\overrightarrow{P}_{\mu})(\lambda\bo+h)&\leq&g_{\mu}+\overleftarrow{P}_{\mu}h\n
\lambda\bo+h&\leq& (I-\overrightarrow{P}_{\mu})^{-1}(g_{\mu}+\overleftarrow{P}_{\mu}h)\n
\lambda\bo+h&\leq& g_{\tilde\mu}+P_{\tilde\mu}h\nonumber
\eea
for all proper policies $\mu$ and all $\tilde\mu\in \mathbb M_{\tilde\mu}$.  Hence, $\tilde\mu^{\star}$ satisfies \eqref{eq:bellmaninmatrix} and $\mu^{\star}$ is optimal over all proper policies.  Using Prop. \ref{prop:optimalpolicyIsStationary}, the proof is complete.
\end{proof}

We will present an algorithm that uses Prop. \ref{prop:optimalityconditiongeneral} to find the optimal policy in the next section.  Note that, unlike \eqref{eq:bellmaninmatrix}, \eqref{eq:bellmanpercycle} can be checked for any policy $\mu$ by finding the minimum for all states $i=1,\ldots,n$, which is significantly easier than searching over all proper policies.


\section{Synthesizing the optimal policy under LTL constraints}
\label{sec:solveLTLprob}
In this section we outline an approach for Prob. \ref{prob:mainprob}.  
We aim for a computational framework, which in combination with results of \cite{IFAC2011_LTL} produces a policy that both maximizes the satisfaction probability and optimizes the long-term performance of the system, using results from Sec. \ref{sec:solutiontoaveragecycle}.

\subsection{Automata-theoretical approach to LTL control synthesis}
\label{sec:sub:maxprobbackground}
Our approach proceeds by converting the LTL formula $\phi$ to a DRA as defined in Def.~\ref{def:DRA}.  We denote the resulting DRA as $\mathcal R_{\phi}=(Q,2^{\Pi},\delta,q_{0},F)$ with $F=\{(L(1),K(1)),\ldots ,(L(M),K(M))\}$ where $L(i),K(i)\subseteq Q$ for all $i=1,\ldots,M$.    We now obtain an MDP as the product of a labeled MDP $\mathcal M$ and a DRA $\mathcal R_{\phi}$, which captures all paths of $\mathcal M$ satisfying $\phi$.
\begin{definition}[Product MDP]
The product MDP $\mathcal M \times \mathcal R_{\phi}$ between a labeled MDP $\mathcal M=(S, U, P, s_{0}, \Pi, \mathcal L, g)$ and a DRA $\mathcal R_{\phi}=(Q,2^{\Pi},\delta,q_{0},F)$ is obtained from a tuple $\mathcal P=(S_{\mathcal P}, U, P_{\mathcal P},s_{\mathcal P0}, F_{\mathcal P}, S_{\mathcal P\pi},g_{\mathcal P})$, where 
\begin{enumerate}
 \item $S_{\mathcal P}= S\times Q$ is a set of states;
 \item $U$ is a set of controls inherited from $\mathcal M$ and we
   define $U_{\mathcal P}((s,q))=U(s)$;
 \item $P_{\mathcal P}$ gives the transition probabilities: 
\ben
P_{\mathcal P}((s,q),u,(s',q'))\! =\!\begin{cases} P(s,u,s') & \textrm{if } q'=\delta(q,\mathcal L(s)) \\ 0 & \rm otherwise; \end{cases}
\een
 \item $s_{\mathcal P0}=(s_{0},q_{0})$ is the initial state; 
 \item $F_{\mathcal P}=\{(L_{\mathcal P}(1), K_{\mathcal P}(1)),\ldots,(L_{\mathcal P}(M), K_{\mathcal P}(M))\}$ where $L_{\mathcal P}(i)=S\times L(i)$, $K_{\mathcal P}(i)=S\times K(i)$, for $i=1,\ldots,M$;
 \item $S_{\mathcal P\pi}$ is the set of states in $S_{\mathcal P}$ for which proposition $\pi$ is satisfied.  Thus, $S_{\mathcal P\pi}=S_{\pi}\times Q$;
 \item $g_{\mathcal P}((s,q),u)=g(s,u)$ for all $(s,q)\in S_{\mathcal P}$;
\end{enumerate}
\end{definition}
Note that some states of $S_{\mathcal P}$ may be unreachable and therefore have no control available.  After removing those states (via a simple graph search), $P$ is a valid MDP and is the desired product MDP.  With a slight abuse of notation we still denote the product MDP as $\mathcal P$ and always assume that unreachable states are removed.   An example of a product MDP between a labeled MDP and a DRA corresponding to the LTL formula $\phi=\gl\ev \pi \andltl \gl\ev a$ is shown in Fig. \ref{fig:MDP-DRA-Product}.  

There is an one-to-one correspondence between a path
$s_{0}s_{1},\ldots$ on $\mathcal M$ and a path
$(s_{0},q_{0})(s_{1},q_{1})\ldots$ on $\mathcal P$.   Moreover, from the
definition of $g_{\mathcal P}$, the costs along these two paths are
the same.  The product MDP is constructed so that, given a path $(s_{0},q_{0})(s_{1},q_{1})\ldots$, the corresponding path $s_{0}s_{1}\ldots$ on $\mathcal M$ generates a word satisfying $\phi$ if and only if, there exists $(L_{\mathcal P}, K_{\mathcal P})\in F_{\mathcal P}$ such that the set $K_{\mathcal P}$ is visited infinitely often and $L_{\mathcal P}$ finitely often.


A similar one-to-one correspondence exists for policies:
\begin{definition}[Inducing a policy from $\mathcal P$]
Given policy $M_{\mathcal P}=\{\mu_{0}^{\mathcal P},\mu_{1}^{\mathcal P},\ldots\}$ on $\mathcal P$, where $\mu^{\mathcal P}_{k}((s,q))\in U_{\mathcal P}((s,q))$, it induces policy $M=\{\mu_{0},\mu_{1},\ldots\}$ on $\mathcal M$ by setting $\mu_{k}(s_{k})=\mu_{k}^{\mathcal P}((s_{k},q_{k}))$ for all $k$.   We denote $M_{\mathcal P}|_{\mathcal M}$ as the policy induced by $M_{\mathcal P}$, and we use the same notation for a set of policies.  
\end{definition}

An induced policy can be implemented on $\mathcal M$ by simply keeping track of its current state on $\mathcal P$.  Note that a stationary policy on $\mathcal P$ induces a non-stationary policy on $\mathcal M$. 
From the one-to-one correspondence between paths and the equivalence of their costs, 
the expected cost in \eqref{eq:averagecostpercycle} from initial state $s_{0}$ under $M_{\mathcal P}|_{\mathcal M}$ is equal to the expected cost from initial state $(s_{0},q_{0})$ under $M_{\mathcal P}$.



For each pair of states $(L_{\mathcal P},K_{\mathcal P})\in F_{\mathcal P}$, we can obtain a set of accepting maximal end components (AMEC):
\begin{definition}[Accepting Maximal End Components] Given $(L_{\mathcal P},K_{\mathcal P})\in F_{\mathcal P}$,
an end component $\mathcal C$ is a communicating MDP $(S_{\mathcal C},U_{\mathcal C}, P_{\mathcal C}, K_{\mathcal C}, S_{\mathcal C\pi}, g_{\mathcal C})$ such that $S_{\mathcal C}\subseteq S_{\mathcal P}$, $U_{\mathcal C}\subseteq U_{\mathcal P}$, $U_{\mathcal C}(i) \subseteq U(i)$ for all $i\in S_{\mathcal C}$, $K_{\mathcal C}=S_{\mathcal C}\cap K_{\mathcal P}$, $S_{\mathcal C\pi}=S_{\mathcal C}\cap S_{\mathcal P\pi}$, and $g_{\mathcal C}(i,u)=g_{\mathcal P}(i,u)$ if $i\in S_{\mathcal C}$, $u\in U_{\mathcal C}(i)$. If $P(i,u,j)>0$ for any $i\in S_{\mathcal C}$ and $u\in U_{\mathcal C}(i)$, then $j\in S_{\mathcal C}$, in which case $P_{\mathcal C}(i,u,j)=P(i,u,j)$.  An accepting maximal end components (AMEC) is the largest such end component such that $K_{\mathcal C}\neq \emptyset$ and $S_{\mathcal C}\cap L_{\mathcal P}= \emptyset$.
\end{definition}
Note that, an AMEC always contains at least one state in $K_{\mathcal P}$ and no state in $L_{\mathcal P}$.   Moreover, it is ``absorbing'' in the sense that the state does not leave an AMEC once entered.  
In the example shown in Fig. \ref{fig:MDP-DRA-Product}, there exists only one AMEC corresponding to $(L_{\mathcal P}, K_{\mathcal P})$, which is the only pair of states in $F_{\mathcal P}$, and the states of this AMEC are shown in Fig. \ref{fig:AMEC}.

\begin{figure}[ht]
   \center
         \includegraphics[scale=0.35]{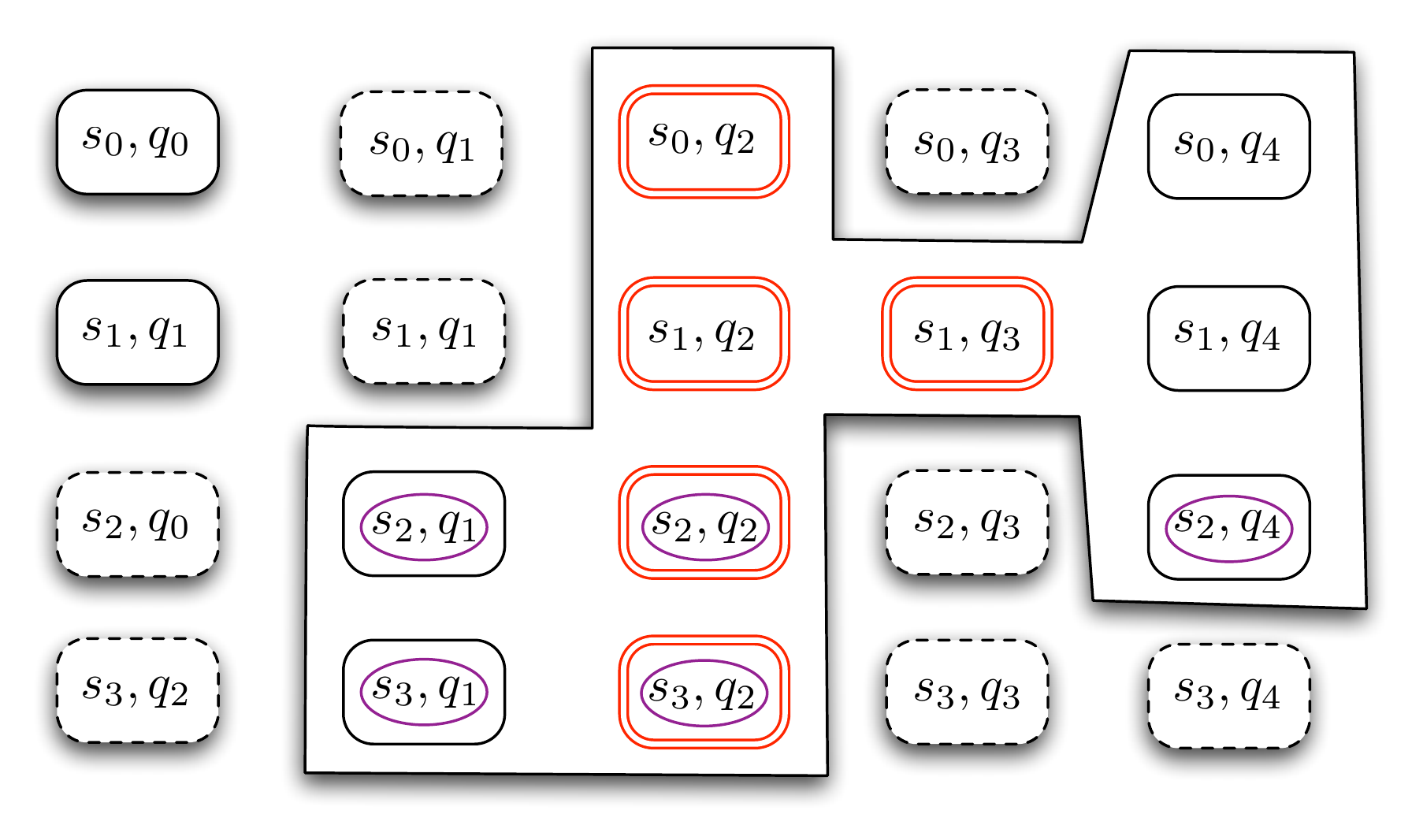}
\caption{The states of the only AMEC corresponding to the product MDP in Fig. \ref{fig:MDP-DRA-Product}.}
\label{fig:AMEC}
\end{figure}

A procedure to obtain all AMECs of an MDP was provided in \cite{baier2008principles}.  
From probabilistic model checking, a policy $M=M_{\mathcal P}|_{\mathcal M}$ almost surely satisfies $\phi$ (\ie $M\in\mathbb M_{\phi}$) if and only if, under policy $M_{\mathcal P}$, there exists AMEC $\mathcal C$ such that the probability of reaching $S_{\mathcal C}$ from initial state $(s_{0},q_{0})$ is $1$ (in this case, we call $\mathcal C$ a \emph{reachable} AMEC).
In \cite{IFAC2011_LTL}, such an optimal policy was found by dynamic programming or solving a linear program.  For states inside $\mathcal C$, since $\mathcal C$ itself is a communicating MDP, a policy (not unique) can be easily constructed such that a state in $K_{\mathcal C}$ is infinitely often visited, satisfying the LTL specification.


\subsection{Optimizing the long-term performance of the MDP}
\label{sec:sub:longtermoptmize}
For a control policy designed to satisfy an LTL formula, the system behavior outside an AMEC is transient, while the behavior inside an AMEC is long-term.    The policies obtained in  \cite{IFAC2011_LTL,baier2004controller,courcoubetis1998markov} essentially disregard the behavior inside an AMEC, because,
from the verification point of view, the behavior inside an AMEC is for the most part irrelevant, as long as a state in $K_{\mathcal P}$ is visited infinitely often.  
We now aim to optimize the long-term behavior of the MDP with respect to the ACPC cost function, while enforcing the satisfaction constraint.  Since each AMEC is a communicating MDP, we can use results in Sec. \ref{sec:sub:optimalACPC} to help obtaining a solution.  Our approach consists of the following steps:
\begin{enumerate}
 \item Convert formula $\phi$ to a DRA $\mathcal R_{\phi}$ and obtain the product MDP $\mathcal P$ between $\mathcal M$ and $\mathcal R_{\phi}$;
 \item Obtain the set of reachable AMECs, denoted as $\mathcal A$;
 \item For each $\mathcal C\in \mathcal A$: Find a stationary policy $\mu^{\star}_{\rightarrow \mathcal C}(i)$, defined for $i\in S\setminus S_{\mathcal C}$, that reaches $S_{\mathcal C}$ with probability $1$ ($\mu^{\star}_{\rightarrow \mathcal C}$ is guaranteed to exist and obtained as in \cite{IFAC2011_LTL}); Find a stationary policy $\mu^{\star}_{\circlearrowright\mathcal C}(i)$, defined for $i\in S_{\mathcal C}$ minimizing \eqref{eq:averagecostpercycle} for MDP $\mathcal C$ and set $S_{\mathcal C\pi}$ while satisfying the LTL constraint; Define $\mu^{\star}_{\mathcal C}$ to be:
 \be
\label{eq:optimalpolicy}
\mu^{\star}_{\mathcal C}=\left\{\ba{cc} \mu^{\star}_{\rightarrow\mathcal C}(i) & \text{if } i\notin S_{\mathcal C}\\ \mu^{\star}_{\circlearrowright\mathcal C}(i) & \text{if } i\in S_{\mathcal C}\ea\right.,
\ee
and denote the ACPC of $\mu^{\star}_{\circlearrowright\mathcal C}$ as $\lambda_{\mathcal C}$;
 \item We find the solution to Prob. \ref{prob:mainprob} by:
\be
\label{eq:optimalcostoverall}
J^{\star}(s_{0})=\min_{\mathcal C\in \mathcal A}\lambda_{\mathcal C},
\ee
and the optimal policy is $\mu^{\star}_{{\mathcal C}^{\star}}|_{\mathcal M}$, where $\mathcal C^{\star}$ is the AMEC attaining the minimum in \eqref{eq:optimalcostoverall}.
\end{enumerate}

We now provide the sufficient conditions for a policy $\mu^{\star}_{\circlearrowright\mathcal C}$ to be optimal.    Moreover, if an optimal policy $\mu^{\star}_{\circlearrowright\mathcal C}$ can be obtained for each $\mathcal C$, we show that the above procedure indeed gives the optimal solution to Prob. \ref{prob:mainprob}. 
\begin{proposition}
\label{prop:optimalPolicyOneAMEC}
For each $\mathcal C\in \mathcal A$, let $\mu^{\star}_{\mathcal C}$ to
be constructed as in \eqref{eq:optimalpolicy}, where
$\mu^{\star}_{\circlearrowright\mathcal C}$ is a stationary policy
satisfying two optimality conditions: (i) its ACPC gain-bias pair is
$(\lambda_{\mathcal C}\bo,h)$, where 
\begin{multline}
\label{eq:bellmanpercycleAMEC}
\lambda_{\mathcal C}+h(i)=\min_{u\in U_{\mathcal
    C}(i)}\bigg[g_{\mathcal C}(i,u)+\sum_{j\in S_{\mathcal C}}
P(i,u,j)h(j) \\+\lambda_{\mathcal C}\sum_{j\notin S_{\mathcal C\pi}} P(i,u,j)\bigg],
\end{multline}
for all $i\in S_{\mathcal C}$, and (ii) there exists a state of $K_{\mathcal C}$ in each recurrent class of $\mu^{\star}_{\circlearrowright\mathcal C}$.  
Then the optimal cost for Prob. \ref{prob:mainprob} is $J^{\star}(s_{0})=\min_{\mathcal C\in \mathcal A}\lambda_{\mathcal C}$,  and the optimal policy is $\mu^{\star}_{{\mathcal C}^{\star}}|_{\mathcal M}$, where $\mathcal C^{\star}$ is the AMEC attaining this minimum.
\end{proposition}
\begin{proof}
Given $\mathcal C\in \mathcal A$, define a set of policies $\mathbb M_{\mathcal C}$, such that for each policy in $\mathbb M_{\mathcal C}$: from initial state  $(s_{0}, q_{0})$, (i) $S_{\mathcal C}$ is reached with probability $1$, (ii) $S\setminus S_{\mathcal C}$ is not visited thereafter, and (iii) $K_{\mathcal C}$ is visited infinitely often.   We see that, by the definition of AMECs, a policy almost surely satisfying $\phi$ belongs to $\mathbb M_{\mathcal C}|_{\mathcal M}$ for a $\mathcal C\in \mathcal A$.  Thus, $\mathbb M_{\phi}=\cup_{\mathcal C\in\mathcal A}\mathbb M_{\mathcal C}|_{\mathcal M}$

Since $\mu_{\mathcal C}^{\star}(i)=\mu^{\star}_{\rightarrow\mathcal C}(i)$ if $i\notin S_{\mathcal C}$, the state reaches $S_{\mathcal C}$ with probability $1$ and in a finite number of stages.   We denote the probability that $j\in S_{\mathcal C}$ is the first state visited in $S_{\mathcal C}$ when $\mathcal C$ is reached from initial state $s_{\mathcal P0}$ as $\widetilde P_{\mathcal C}(j, \mu^{\star}_{\rightarrow\mathcal C}, s_{\mathcal P0})$.   Since the ACPC for the finite path from the initial state to a state $j\in S_\mathcal C$ is $0$ as the cycle index goes to $\infty$, the ACPC from initial state $s_{\mathcal P0}$ under policy $\mu^{\star}_{\mathcal C}$ is
\be
J(s_{0})=\sum_{j\in S_{\mathcal C}}\widetilde P_{\mathcal C}(j, \mu^{\star}_{\rightarrow\mathcal C}, s_{\mathcal P0}) J_{\mu^{\star}_{\circlearrowright\mathcal C}}(j).
\ee
Since $\mathcal C$ is communicating, the optimal cost is the same for all states of $S_{\mathcal C}$ (and thus it does not matter which state in $S_{\mathcal C}$ is first visited when $S_{\mathcal C}$ is reached).  We have 
\bea
J(s_{0})&=&\sum_{j\in S_{\mathcal C}}\widetilde P_{\mathcal C}(j, \mu^{\star}_{\rightarrow\mathcal C}, (s_{0},q_{0}))\lambda_{\mathcal C} \n
&=&\lambda_{\mathcal C}.
\eea
Applying Prop. \ref{prop:optimalityconditiongeneral}, we see that $\mu^{\star}_{\circlearrowright\mathcal C}$ satisfies the optimality condition for MDP $\mathcal C$ with respect to set $S_{\mathcal C\pi}$.  
Since there exists a state of $K_{\mathcal C}$ is in each recurrent class of $\mu^{\star}_{\circlearrowright\mathcal C}$, a state in $K_{\mathcal C}$ is visited infinitely often and it satisfies the LTL constraint.   Therefore, $\mu^{\star}_{\mathcal C}$ as constructed in \eqref{eq:optimalpolicy} is optimal over $\mathbb M_{\mathcal C}$ and $\mu^{\star}_{\mathcal C}|_{\mathcal M}$ is optimal over $\mathbb M_{\mathcal C}|_{\mathcal M}$ (due to equivalence of expected costs between $M_{\mathcal P}$ and $M_{\mathcal P}|_{\mathcal M}$).   Since $\mathbb M_{\phi}=\cup_{\mathcal C\in\mathcal A}\mathbb M_{\mathcal C}|_{\mathcal M}$, we have that $J^{\star}(s_{0})=\min_{\mathcal C\in \mathcal A}\lambda_{\mathcal C}$ and the policy corresponding to $\mathcal C^{\star}$ attaining this minimum is the optimal policy.
\end{proof}


We can relax the optimality conditions for $\mu^{\star}_{\circlearrowright\mathcal C}$ in Prop. \ref{prop:optimalPolicyOneAMEC} and require that there exist a state $i\in K_{\mathcal C}$ in one recurrent class of $\mu^{\star}_{\circlearrowright\mathcal C}$.  For such a policy, we can construct a policy such that it has one recurrent class containing state $i$, with the same ACPC cost at each state.   This construction is identical to a similar procedure for ACPS problems when the MDP is communicating (see \cite[p.~203]{bertsekas2007dynamic}).  We can then use \eqref{eq:optimalpolicy} to obtain the optimal policy $\mu^{\star}_{\mathcal C}$ for $\mathcal C$. 


We now present an algorithm (see Alg. \ref{alg:policyIteration}) that iteratively updates the policy in an attempt to find one that satisfies the optimality conditions given in Prop. \ref{prop:optimalPolicyOneAMEC}, for a given $\mathcal C\in \mathcal A$.  
Note that Alg. \ref{alg:policyIteration} is similar in nature to policy iteration algorithms for ACPS problems.  
\begin{algorithm}
\caption{: Policy iteration algorithm for ACPC}
\label{alg:policyIteration}
\begin{algorithmic}[1]
\REQUIRE{$\mathcal C=(S_{\mathcal C},U_{\mathcal C}, P_{\mathcal C}, K_{\mathcal C}, S_{\mathcal C\pi}, g_{\mathcal C})$}
\ENSURE{Policy $\mu_{\circlearrowright\mathcal C}$} 
\STATE Initialize $\mu^{0}$ to a proper policy containing $K_{\mathcal C}$ in its recurrent classes (such a policy can always be constructed since $\mathcal C$ is communicating)
\REPEAT
\STATE Given $\mu^{k}$, compute $J_{\mu^{k}}$ and $h_{\mu^{k}}$ with \eqref{eq:helpereqmu} and \eqref{eq:vmutilde}
\STATE Compute for all $i\in S_{\mathcal C}$:
\be
\label{eq:firstoptimization}
\bar U(i) = \argmin_{u\in U_{\mathcal C}(i)}\sum_{j\in S_{\mathcal C}} P(i,u,j)J_{\mu^{k}}(j)
\ee
\IF{$\mu^{k}(i)\in \bar U(i)$ for all $i\in S_{\mathcal C}$}
\STATE Compute, for all $i\in S_{\mathcal C}$:
\begin{multline}
\label{eq:secondoptimization}
\bar M(i)=\argmin_{u\in \bar U(i)}\bigg[g_{\mathcal
  C}(i,u)+\sum_{j\in S_{\mathcal C}} P(i,u,j)h_{\mu^{k}}(j) \\ 
+\sum_{j\notin S_{\mathcal C\pi}} P(i,u,j)J_{\mu^{k}}(j)\bigg]
\end{multline}

\STATE Find $\mu^{k+1}$ such that $\mu^{k+1}(i)\in \bar M(i)$ for all $i\in S_{\mathcal C}$, and contains a state of $K_{\mathcal C}$ in its recurrent classes.  If one does not exist. {\bf Return:} $\mu^{k}$ with ``not optimal''
\ELSE
\STATE Find $\mu^{k+1}$ such that $\mu^{k+1}(i)\in \bar U(i)$ for all $i\in S_{\mathcal C}$, and contains a state of $K_{\mathcal C}$ in its recurrent classes.  If one does not exist, {\bf Return:} $\mu^{k}$ with ``not optimal''
\ENDIF
\STATE Set $k \leftarrow k+1$
\UNTIL{$\mu^{k}$ with gain-bias pair satisfying \eqref{eq:bellmanpercycleAMEC} and {\bf Return:} $\mu^{k}$ with ``optimal''}
\end{algorithmic}
\end{algorithm}
\begin{proposition}
  Given $\mathcal C$, Alg. \ref{alg:policyIteration} terminates in a finite number of
  iterations.  If it returns policy
  $\mu_{\circlearrowright\mathcal C}$ with ``optimal'', then
  $\mu_{\circlearrowright\mathcal C}$ satisfies the optimality
  conditions in Prop. \ref{prop:optimalPolicyOneAMEC}. If $\mathcal C$
  is unichain (\ie each stationary policy of $\mathcal C$ contains one recurrent class), then Alg. \ref{alg:policyIteration} is guaranteed to
  return the optimal policy $\mu^{\star}_{\circlearrowright\mathcal C}$.
\end{proposition}
\begin{proof}
If $\mathcal C$ is unichain, then since it is also communicating,  $\mu^{\star}_{\circlearrowright\mathcal C}$ contains a single recurrent class (and no transient state).  In this case, since $K_{\mathcal C}$ is not empty, states in $K_{\mathcal C}$ are recurrent and the LTL constraint is always satisfied at step $7$ and $9$ of Alg. \ref{alg:policyIteration}.  The rest of the proof (for the general case and not assuming $C$ to be unichain) is similar to the proof of convergence for the policy iteration algorithm for the ACPS problem (see \cite[pp.~237-239]{bertsekas2007dynamic}).  Note that the proof is the same except that when the algorithm terminates at step $11$ in Alg. \ref{alg:policyIteration}, $\mu^{k}$ satisfies \eqref{eq:bellmanpercycleAMEC} instead of the optimality conditions for the ACPS problem (\eqref{eq:bellmanacps1} and \eqref{eq:bellmanacps2}).
\end{proof}

If we obtain the optimal policy for each $\mathcal C\in \mathcal A$, then we use \eqref{eq:optimalcostoverall} to obtain the optimal solution for Prob. \ref{prob:mainprob}.  If for some $\mathcal C$, Alg. \ref{alg:policyIteration} returns ``not optimal'', then the policy returned by Alg. \ref{alg:policyIteration} is only sub-optimal.  We can then apply this algorithm to each AMEC in $\mathcal A$ and use \eqref{eq:optimalcostoverall} to obtain a sub-optimal solution for Prob. \ref{prob:mainprob}.  Note that similar to policy iteration algorithms for ACPS problems, either the gain or the bias strictly decreases every time when $\mu$ is updated, so policy $\mu$ is improved in each iteration.   In both cases, the satisfaction constraint is always enforced.
%

\begin{remark}[Complexity]
The complexity of our proposed algorithm is dictated by the size of the generated MDPs.  We use $|\cdot|$ to denote cardinality of a set.  The size of the DRA ($|Q|$) is in the worst case, doubly exponential with respect to $|\Sigma|$.  However, empirical studies such as \cite{klein2006experiments} have shown that in practice, the sizes of the DRAs for many LTL formulas are generally much lower and manageable. The size of product MDP $\mathcal P$ is at most $|\mathcal S|\times|Q|$.  The complexity for the algorithm generating AMECs is at most quadratic in the size of $\mathcal P$ \cite{baier2008principles}.  The complexity of Alg. \ref{alg:policyIteration} depends on the size of $\mathcal C$.   The policy evaluation (step $3$) requires solving a system of $3\times |S_{\mathcal C}|$ linear equation with $3\times |S_{\mathcal C}|$ unknowns.  The optimization step (step $4$ and $6$) each requires at most $|U_{\mathcal C}|\times |S_{\mathcal C}|$ evaluations.  Checking the recurrent classes of $\mu$ is linear in $|S_{\mathcal C}|$.  Therefore, assuming that $|U_{\mathcal C}|$ is dominated by $|S_{\mathcal C}|^{2}$ (which is usually true) and the number of policies satisfying \eqref{eq:firstoptimization} and \eqref{eq:secondoptimization} for all $i$ is also dominated by $|S_{\mathcal C}|^{2}$, for each iteration, the computational complexity is $O(|S_{\mathcal C}|^{3})$.
\end{remark}

\section{Case study}
\label{sec:casestudy}
The algorithmic framework developed in this paper is implemented in
MATLAB, and here we provide an example as a case study.  Consider the
MDP $\mathcal M$ shown in Fig.~\ref{fig:MDPexample}, which can be
viewed as the dynamics of a robot navigating in an environment with
the set of atomic propositions $\{\mathtt{pickup}
,\mathtt{dropoff}\}$.  In practice, this MDP can be obtained via an abstraction process (see \cite{LaWaAnBe-ICRA10}) from the environment, where its probabilities of transitions can be obtained from experimental data or accurate simulations.  
\begin{figure}[h]
\begin{center}
\includegraphics[scale=0.35]{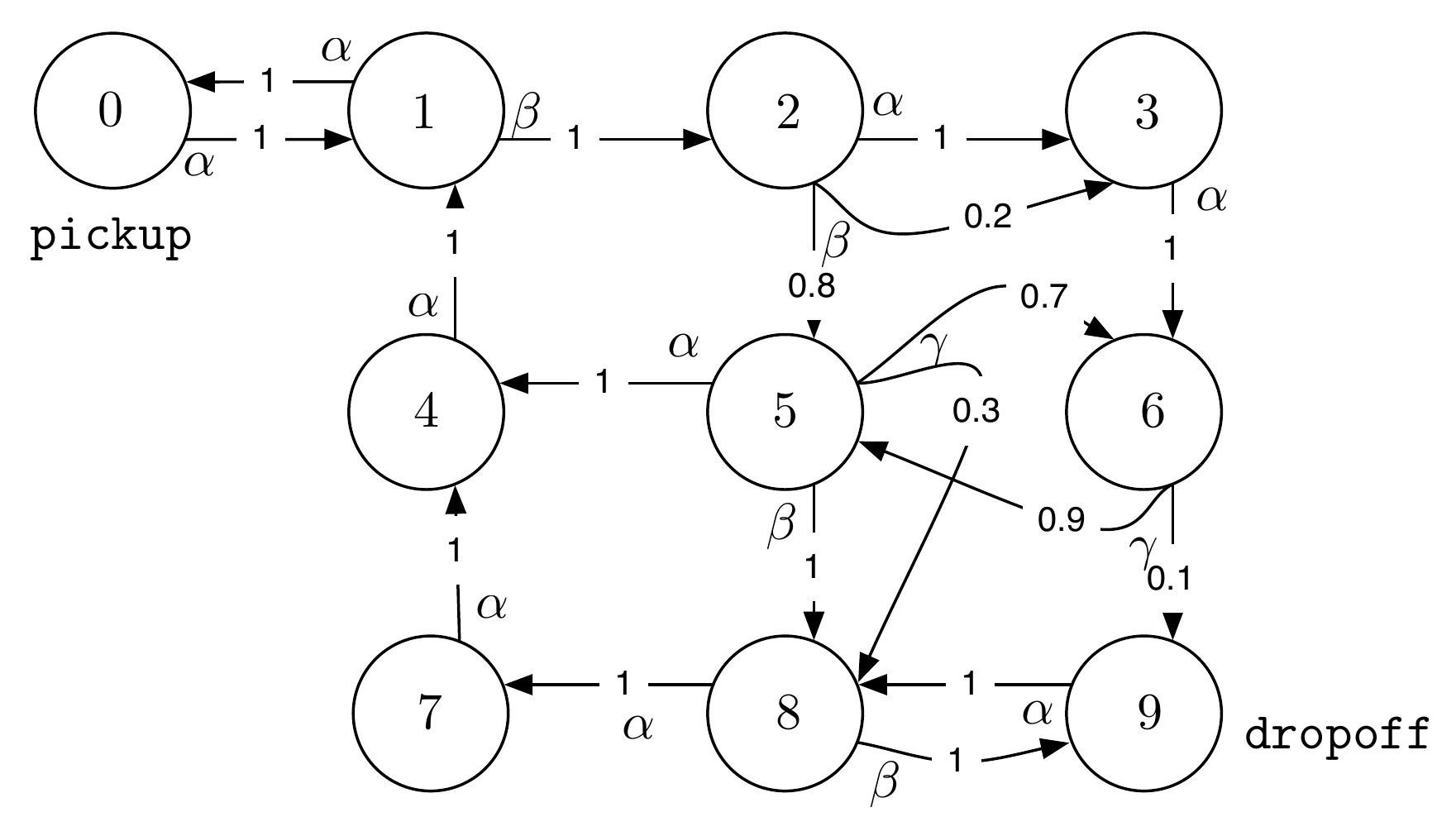}
\caption{MDP capturing a robot navigating in an environment.  $\{\alpha, \beta, \gamma\}$ is the set of controls at states.  The cost of applying $\alpha, \beta, \gamma$ at a state where the control is available is $5,10,1$, respectively. (\eg $g(i,\alpha)=5$ if $\alpha\in U(i)$)}%
\label{fig:MDPexample}
\end{center}
\end{figure}

The goal of the robot is to continuously
perform a pickup-delivery task.  The robot is required to pick up
items at the state marked by $\mathtt{pickup}$ (see
Fig. \ref{fig:MDPexample}), and drop them off at the state marked by
$\mathtt{dropoff}$.  It is then required to go back to
$\mathtt{pickup}$ and this process is repeated.
This task can be written as the following LTL formula:
\ben
\label{eq:exampleformula}
\phi=\gl\ev \mathtt{pickup} \andltl \gl (\mathtt{pickup} \Rightarrow \nextltl (\notltl \mathtt{pickup} \un \mathtt{dropoff})).
\een
The first part of $\phi$, $\gl\ev \mathtt{pickup}$, enforces that the robot repeatedly pick up items.  The remaining part of $\phi$ ensures that new items cannot be picked up until the current items are dropped off.  We denote $\mathtt{pickup}$ as the optimizing proposition, and the goal is to find a policy that satisfies $\phi$ with probability $1$ and minimizes the expected cost in between visiting the $\mathtt{pickup}$ state (\ie we aim to minimize the expected cost in between picking up items).

We generated the DRA $\mathcal R_{\phi}$ using the ltl2dstar tool
\cite{ltl2dstar} with $13$ states and $1$ pair $(L,K)\in F$.  The
product MDP $\mathcal P$ after removing unreachable states contains
$31$ states (note that $\mathcal P$ has $130$ states without removing
unreachable states).  There is one AMEC $\mathcal C$ corresponding to
the only pair in $F_{\mathcal P}$ and it contains $20$ states.  We
tested Alg. \ref{alg:policyIteration} with a number of different
initial policies and Alg. \ref{alg:policyIteration} produced the
optimal policy within $2$ or $3$ policy updates in each case (note
that $\mathcal C$ is not unichain).  For one initial policy, the ACPC
was initially $330$ at each state of $\mathcal C$, and it was reduced
to $62.4$ at each state when the optimal policy was found.  The
optimal policy
is as follows: 
\begin{tabular}[c]{|@{ }c@{ }|l|l|l|l|l|l|l|l|l|l|}
\noalign{\smallskip} \hline 
\small State & \small 0 & \small 1& \small 2 & \small 3 & \small 4 & \small 5 & \small 6 & \small 7& \small 8 & \small 9 \\
\hline
\small After $\mathtt{pickup}$ & \small $\alpha$ & \small $\beta$ & \small $\alpha$ & \small $\alpha$ & \small $\alpha$ &
\small $\gamma$ & \small $\gamma$ & \small $\alpha$ & \small $\beta$ & \small $\alpha$ \\
\hline
\small After $\mathtt{dropoff}$ & \small $\alpha$ & \small $\alpha$ & \small $\alpha$ & \small $\alpha$ & \small $\alpha$ &
\small $\alpha$ & \small $\gamma$ & \small $\alpha$ & \small $\alpha$ & \small $\alpha$\\
\hline \noalign{\smallskip}
\end{tabular}
\noindent The first row of the above table shows the policy after pick-up but before drop-off and the second row shows the policy after drop-off and before another pick-up. 

\section{Conclusions}
\label{sec:conc}

We have developed a method to automatically generate a control policy
for a dynamical system modelled as a Markov Decision Process (MDP), in
order to satisfy specifications given as Linear Temporal Logic
formulas. The control policy satisfies the given specification almost
surely, if such a policy exists. In addition, the policy optimizes the
average cost between satisfying instances of an ``optimizing
proposition'', under some conditions. The problem is motivated by robotic applications
requiring persistent tasks to be performed such as environmental
monitoring or data gathering.

We are currently pursuing several future directions.  First, we aim to solve the problem completely and find an algorithm that guarantees to always return the optimal policy.   Second, we are interested to apply the optimization criterion of average cost per cycle to more complex models such as Partially Observable MDPs (POMDPs) and semi-MDPs.  


\end{document}